\documentclass[letterpaper, 10 pt, conference]{ieeeconf}  % Comment this line out if you need a4paper

\IEEEoverridecommandlockouts                              % This command is only needed if 
\makeatletter
\let\NAT@parse\undefined
\makeatother
\usepackage[numbers,sort&compress]{natbib}

\usepackage{amsmath,amsfonts}
\usepackage{algorithmic}
\usepackage{array}
\usepackage{textcomp}
\usepackage{stfloats}
\usepackage{url}
\usepackage{verbatim}
\usepackage{graphicx}
\usepackage[utf8]{inputenc}

\usepackage{mathtools}
\usepackage{hyperref}
\usepackage{booktabs}       % professional-quality tables
\usepackage{lineno, color, xcolor}

\usepackage[font=footnotesize]{subcaption}
\usepackage{caption}
\usepackage[capitalise]{cleveref}

\usepackage{amssymb}

\usepackage{amsthm}
\newtheorem{proposition}{Proposition}
\newtheorem{corollary}{Corollary}[proposition]

\usepackage{algorithm}

\newcommand*{\Object}{O}
\newcommand*{\Center}{c}
\newcommand*{\WorkSpace}{\mathcal{W}}
\newcommand*{\FreeSpace}{\mathcal{W}_\text{free}}
\newcommand*{\ObstacleSet}{\mathcal{O}_\text{obs}}

\newcommand*{\VisGraph}{G_v}
\newcommand*{\CoarseGraph}{G_c}
\newcommand*{\DenseGraph}{G_d}

\newcommand*{\VisVert}{\mathcal{V}_v}
\newcommand*{\VisEdge}{\mathcal{E}_v}
\newcommand*{\CoarseVert}{\mathcal{V}_c}
\newcommand*{\CoarseEdge}{\mathcal{E}_c}
\newcommand*{\DenseVert}{\mathcal{V}_d}
\newcommand*{\DenseEdge}{\mathcal{E}_d}

\newcommand*{\IRISSet}{\mathcal{P}}

\newcommand*{\CVertU}{P_i}
\newcommand*{\CVertW}{P_j}
\newcommand*{\CIntersection}{P_{ij}}
\newcommand*{\SurfaceInter}{\partial P_{ij}}
\newcommand*{\DVertU}{v_{ij,n_1}}
\newcommand*{\DVertW}{v_{ik,n_2}}

\title{\LARGE \bf
Rigid Body Path Planning using Mixed-Integer Linear Programming
}

\author{Mingxin Yu$^{1}$ and Chuchu Fan$^{1}$% <-this % stops a space
\thanks{$^{1}$ The authors are with the Department of Aeronautics and Astronautics, Massachusetts Institute of Technology, Cambridge, MA 02139, USA. {\tt\small \{yumx35, chuchu\}@mit.edu}}%
}

\begin{document}

% \begin{onecolumn}
%     \setcounter{section}{0}
%     \setcounter{figure}{6}
%     \input{RAL-revision/B-response}
% \end{onecolumn}

% \twocolumn\newpage
\maketitle
% \thispagestyle{empty}
% \pagestyle{empty}

% \setcounter{section}{0}
% \setcounter{figure}{0}

%%%%%%%%%%%%%%%%%%%%%%%%%%%%%%%%%%%%%%%%%%%%%%%%%%%%%%%%%%%%%%%%%%%%%%%%%%%%%%%%
\begin{abstract}
Navigating rigid body objects through crowded environments can be challenging, especially when narrow passages are presented. Existing sampling-based planners and optimization-based methods like mixed integer linear programming (MILP) formulations, suffer from limited scalability with respect to either the size of the workspace or the number of obstacles.
In order to address the scalability issue, we propose a three-stage algorithm that first generates a graph of convex polytopes in the workspace free of collision, then poses a large set of small MILPs to generate viable paths between polytopes, and finally queries a pair of start and end configurations for a feasible path online.
The graph of convex polytopes serves as a decomposition of the free workspace and the number of decision variables in each MILP is limited by restricting the subproblem within two or three free polytopes rather than the entire free region. 
Our simulation results demonstrate shorter online computation time compared to baseline methods and scales better with the size of the environment and tunnel width than sampling-based planners in both 2D and 3D environments.

\end{abstract}

%%%%%%%%%%%%%%%%%%%%%%%%%%%%%%%%%%%%%%%%%%%%%%%%%%%%%%%%%%%%%%%%%%%%%%%%%%%%%%%%
\section{Introduction}
The \textit{Piano Mover's Problem}~\cite{schwartz1983piano} asks whether one can find a sequence of rigid body motions from a given initial position to a desired final position, subject to certain geometric constraints during the motion. 
The difficulty in planning for point objects comes from the dimension of the objects, which prevents the feasibility of naively tracking the motion of a point, especially when narrow corridors are presented in the environment and an associated higher-dimension configuration space of $SE(2)$ or $SE(3)$ rather than $\mathbb R^2$ or $\mathbb R^3$. 
This problem has inspired many motion planning works~\cite{kavraki1996prm, tomlin2008tunnel_milp, dobson2014sparse, nister2023nonholonomic}, mainly streamed as sampling-based and optimization-based methods. 

\begin{figure}
    \centering
    \includegraphics[width=.7\linewidth]{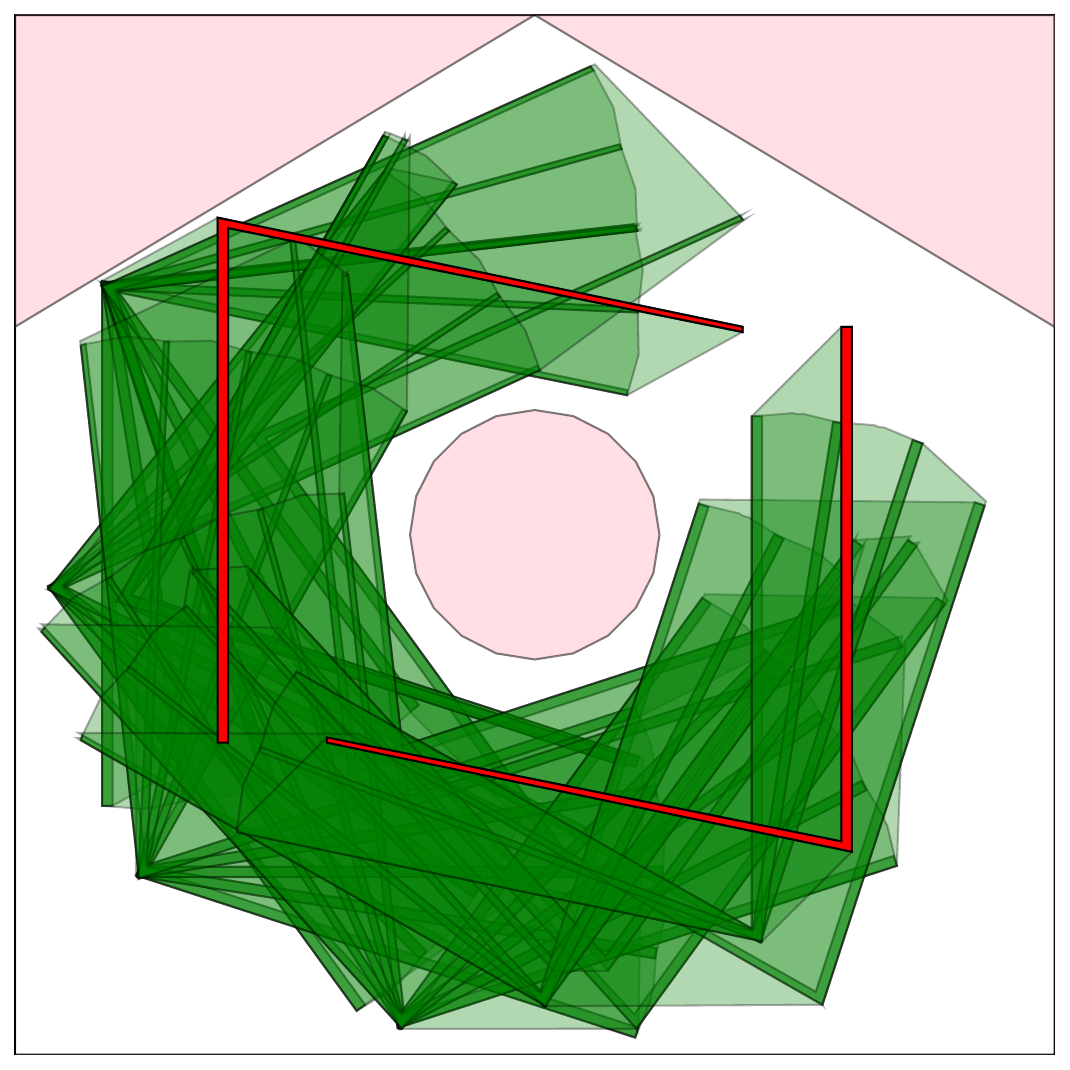}
    \caption{The figure demonstrates an example solved with our MILP-based method, navigating a non-convex, V-shaped object through a complex environment from the start configuration to the goal configuration, both marked red. In the figure, the obstacles are shown in pink, while the green depicts the regions swept by the object, demonstrating the effectiveness and precision of our approach in challenging environments.}
    \label{fig:enter-label}
\end{figure}

% sampling-based method
Sampling-based methods like probabilistic roadmap (PRM)~\cite{kavraki1996prm} are widely adopted because of their simplicity. The approaches employ discrete collision detection, which is straightforward and efficient but highly dependent on the chosen resolution. Setting improper parameters can lead to missed obstacles, resulting in invalid paths.
% optimization (MILP & gcs)
On the other side, solving an obstacle avoidance problem can be formulated as optimization problems for optimality and safety guarantee, such as trajectory optimization using Mixed-Integer Linear Programming (MILP)~\cite{tomlin2008tunnel_milp,8613017} or Graphs of Convex Sets (GCS)~\cite{amice2022ciris}. They offer a more robust solution by directly incorporating collision avoidance into the formulations. These methods, unlike sampling-based methods, ensure paths being collision-free by design, thus bypassing the dependency of resolution in collision detection.

% why our optimization formulation
Despite their advantages, current optimization formulations come with their own set of challenges. 
Solving MILPs, for example, is an NP-hard problem and can be computationally intensive as the environment complexity increases. The time to directly solve a MILP for a feasible path grows exponentially with the number of waypoints, which is proportional to the required maneuvers in an environment. This scalability issue makes brute-force MILP-based methods difficult to apply.
On the other hand, the effectiveness of GCS depends on precise initial seeding to capture narrow passages within the configuration space, which brings feasibility challenges to GCS methods. 
To address these inefficiencies, our approach simplifies the problem by decomposing a single large MILP into manageable, fixed-size smaller MILPs. Additionally, by focusing directly on the workspace rather than the configuration space, our method mitigates the challenges associated with identifying critical narrow passages, thereby enhancing efficiency and scalability in path planning.

In this work, our pipeline is structured into three key stages: 
\textit{Workspace decomposition} - we first construct a graph of convex set $\CoarseGraph$ covering the free workspace. This stage is performed offline and the graph is reusable for multiple queries across different objects. We exploit the lower dimensionality of the workspace compared to the configuration space to obtain a higher coverage ratio and fewer narrow tunnels, which are typically challenging to identify and cover.
\textit{Path segment validation} - we then construct a graph $\DenseGraph$ in configuration space, which is also offline. Here, the nodes represent bottleneck configurations identified by $\CoarseGraph$ and the edges are viable paths validated through MILPs. This graph allows multiple queries for different start and goal configurations. 
At the \textit{online} stage, we connect the start and goal configuration into $\DenseGraph$ and retrieve the solution path. Our experiment results show that this method significantly reduces online time consumption across various objects and environments, in both 2D and 3D settings. 

\begin{figure*}[thb]
    \centering
    \includegraphics[width=0.91\textwidth]{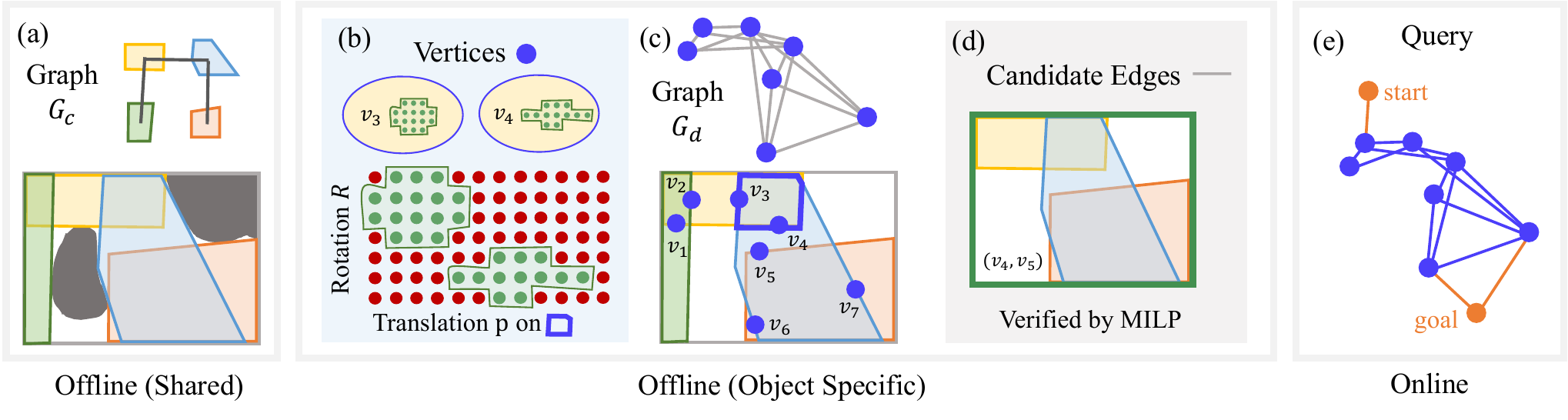}
    \caption{An overview of our pipeline. (a) \textit{Workspace decomposition:} Decomposition of the free workspace into a graph of convex polytopes $\CoarseGraph$. The polytopes serve as graph vertices and are interconnected if they overlap  (\cref{sec:coarse_graph}). 
    (b-d) \textit{Path segment validation:} Construction of $\DenseGraph$(\cref{sec:dense_graph}), where each vertex is a set of interconnected free configurations - illustrated as the set of green points enclosed by the green ring. The motions between vertices (blue edges) are generated by solving MILPs (\cref{sec:dense-node-full}).
    (e) \textit{Online query}: the start and end configurations are connected to graph $\DenseGraph$ and the planned path is retrieved (\cref{sec:inference}).}
    \label{fig:pipeline}
\end{figure*}

This work presents several key contributions: 
(1) By operating directly within the workspace, our method effectively mitigates the challenges associated with covering narrow tunnels in the free space, compared to configuration space-based approaches.
(2) We simplify the path planning process by breaking down a large, complex optimization problem into a series of smaller MILPs. Each smaller problem focuses on finding a valid path segment within a region, significantly enhancing the scalability of our approach.
(3) We conduct comprehensive experiments in both 2D and 3D environments, comparing against GCS and sampling-based methods. Our results demonstrate scalability with the environment, higher feasibility, and shorter computation time.

% \vspace{-6pt}
\section{Related Works}\label{sec:related_work}
\noindent\textbf{Graph-based methods.} 
Graph-based methods are particularly suitable for scenarios requiring multiple queries within the same environment. Once the graph is constructed, it can be reused to find multiple paths.
One category builds the graph of discrete configurations, like sampling-based methods PRM~\cite{kavraki1996prm} and lattice planners~\cite{nister2023nonholonomic}. 
PRM randomly samples states in C-space as nodes in the graph. Nodes are connected based on the feasibility of direct paths. Some variants of PRM also consider utilizing workspace information to guide sampling in configuration space~\cite{kurniawati2008workspace, 5477241}, therefore accelerating the exploration.
Lattice planners are commonly used in nonholonomic vehicle parking scenarios and leverage high levels of parallelization~\cite{mcnaughton2011parallel,nister2023nonholonomic}. They use a regular grid in the configuration space to define graph nodes~\cite{kelly2006toward}. And the edges are computed offline~\cite{pivtoraiko2009differentially}, representing motions that adhere to specific constraints. However, one inherent limitation of lattice planners is their scaling with grid resolution. Employing regular lattices across the entire configuration space can be computationally expensive and inefficient. In contrast, our method reduces the size of the graph by only utilizing a regular grid on a set of manifolds within the configuration space, where the metric is strictly zero.

Another category of graph-based planners relies on space decomposition, where the graph explicitly represents the free space.
Based on constrained Delaunay triangulation~\cite{chew1987constraineddelaunay}, previous works perform well in 2D environments for point objects~\cite{demyen2006efficientTriangulation,tomlin2008tunnel_milp} either through search or by solving MILPs. Some works have addressed the problem for higher DoF robots by incorporating potential field after decomposition for collision avoidance~\cite{932817} in less-cluttered 3D environments. While these approaches provide a foundation for solving the problem, they may struggle with highly constrained environments or complex object geometries.
GCS line of works~\cite{deits2015iris,amice2022ciris,marcucci2024spp-gcs} adopts a different approach by attempting to cover the free space with a set of convex polytopes, thereby retrieving paths through optimization. This method is capable of working in both workspace and configuration space, thus accommodating any shape and scaling to more than 10 dimensions.

\noindent\textbf{Exact collision detection}
Most path-planning algorithms work with discrete collision detection, due to their easy implementation, broad applicability, and fast execution~\cite{kavraki1996prm,lavalle2001randomized}. This method checks a path in C-Space by creating samples along the path and when all these samples are checked to be collision-free, the entire path is assumed to be collision-free. While efficient, its accuracy depends on the sampling resolution, risking missed collisions with coarse samples or delays with overly fine-grained samples. 
To enhance reliability, one method is free bubble~\cite{quinlan1995real}, recursively bisecting the samples on the path to guarantee collision-free~\cite{schwarzer2005adaptive} by calculating the maximum allowable movement without collision. 
Another method is continuous collision checking~\cite{845313,zhang2007continuous}, which performs well for rigid body motions by directly checking the collisions of the reachable set with the obstacles. But they suffer from the speed. 
We integrate both methods in the algorithm, aiming for both soundness and efficiency.

% \vspace{-6pt}
\section{Pipeline Overview}\label{sec:pipeline}
In this paper, we consider the path planning problem of a rigid body object $\Object$. The workspace $\WorkSpace$, the physical space where the object lies, can be $\mathbb R^2$ or $\mathbb R^3$ depending on the particular application at hand. And the configuration space for $\Object$ is $SE(2)$ and $SE(3)$, respectively. A configuration is represented as $q=(p,R)$ with a position vector $p\in\mathbb R^2$ or $\mathbb R^3$ and a rotation matrix $R\in SO(2)$ or $SO(3)$.
The pipeline of our method is shown in \cref{fig:pipeline}. 
We outline major stages of our method in this section, before going deep into verifying path segments with MILP in \cref{sec:non-point}.

\subsection{Workspace decomposition - Construct Coarse Graph $\CoarseGraph$}\label{sec:coarse_graph}
\setlength{\textfloatsep}{2pt}{
\begin{algorithm}[!b]
    \caption{Construct a graph of convex polytopes $\CoarseGraph$}\label{alg:iris}
    \begin{algorithmic}[1]
        \REQUIRE list of obstacles $\ObstacleSet$, number of samples for visibility graph $n_v$, number of samples per iteration $n_s$, coverage threshold $\alpha$.
        \STATE $\IRISSet \gets \emptyset$
        \STATE $\VisGraph \gets \textsc{SampleVisibilityGraph}(\ObstacleSet, n_v)$
        \WHILE{$\textsc{CheckCoverage}(\VisEdge, \IRISSet) < \alpha$}
            \STATE $\mathcal S \gets \textsc{SampleVisibilityEdge}(\VisEdge, \IRISSet, n_s)$
            \STATE $\IRISSet \gets \IRISSet\,\cup \textsc{Iris}(\mathcal S)$
        \ENDWHILE
        \STATE $\CoarseGraph \gets \textsc{ConstructGraph}(\IRISSet)$
        \RETURN $\CoarseGraph$
    \end{algorithmic}
\end{algorithm}}
At this stage, we aim to approximately decompose the free workspace $\FreeSpace$ into a set of convex polytopes $\IRISSet$ and build a graph $\CoarseGraph$ on $\IRISSet$ as shown in~\cref{alg:iris}. The graph $\CoarseGraph$ is shared across various objects within the same environment. 

Following~\cite{werner2023irisclique}, we first construct a visibility graph $\VisGraph\coloneq(\VisVert, \VisEdge)$ where $\VisVert$ is uniformly sampled from $\WorkSpace$ with points inside obstacles excluded, and $\VisEdge$ is added by checking for collisions along the line segments connecting each pair of points within some distance using~\cref{thm:line}. 
During each iteration, the subroutine \textsc{SampleVisibilityEdge} samples $n_s$ points on the edges in $\VisEdge$ that are not yet covered by the polytopes in $\IRISSet$. The $n_s$ points are then used as initial points for IRIS~\cite{deits2015iris}, after which the newly computed polytopes are added into $\IRISSet$. 
The iteration terminates when the coverage ratio of $\IRISSet$ over $\VisEdge$ exceeds threshold $\alpha > 0$. 
Subsequently, an undirected graph $\CoarseGraph\coloneq(\CoarseVert, \CoarseEdge)$ is constructed with vertices $\CoarseVert=\IRISSet$, and with an edge $(P_{i_1},P_{i_2})\in\CoarseEdge$ for every pair of intersected polytopes $P_{i_1}$ and $P_{i_2}$.

\subsection{Path Segment Validation - Construct Dense Graph $\DenseGraph$}\label{sec:dense_graph}
Graph $\CoarseGraph$ is a GCS, and an optimal solution of the path planning problem for a point object can be found via~\cite{marcucci2024spp-gcs}. However, its direct applicability is limited when considering non-point objects. 
The limitation arises because an intersection between polytopes - a valid pathway for point objects - does not inherently guarantee feasible traversal for objects with non-negligible dimensions and orientations.

To address this challenge, we introduce another undirected graph $\DenseGraph\coloneq(\DenseVert, \DenseEdge)$ that describes the connectivity between adjacent polytopes for a specific object $\Object$, detailed in~\cref{alg:dense-graph}.
Our key insight is that, for an object $\Object$ to move from $\CVertU$ to $\CVertW$, its center $\Center$ must traverse the boundary $\SurfaceInter$ of the intersection $\CIntersection\coloneq \CVertU\cap \CVertW$. 

\noindent\textbf{Vertices $\DenseVert$.} \label{sec:dense-node-full}
The vertices $\DenseVert$ in $\DenseGraph$ are generated through subroutine \textsc{Sample\&Group}. The process begins with discretizing $\SurfaceInter$ into regular grids. 

\textit{a) Discretization in 2D environments.}
For 2D, the boundary $\SurfaceInter$ is a closed ring, which can be parameterized linearly from $\lambda=0$ to $\lambda=1$, with the points for $\lambda=0$ and $\lambda=1$ being identical. Translations along this boundary are selected with a fixed interval $\delta\lambda> 0$. 
Rotations are discretized into $n_R > 0$ possible values, with angles $(\theta_1, \cdots, \theta_{n_R})$ being $(1\cdot2\pi/n_R, 2\cdot2\pi/n_R,\cdots, n_R\cdot2\pi/n_R)$. These correspond to rotation matrices $(R_1, \cdots, R_{n_R})$. Together, the translation parameter $\lambda$ and the set of rotations ${\theta}$ create a 2D grid of possible configurations, as shown in~\cref{fig:pipeline}.

\textit{b) Discretization in 3D environments.}
In 3D scenarios, the boundary $\SurfaceInter$ consists of several facets, each treated as a disjoint boundary without considering the connections between them.
On each facet, the translations are constructed using a regular rectangle grid, and the rotations are a selected set of rotation matrices $(R_1, \cdots, R_{n_R})$. Together, the translations and rotations create a 3D grid on a facet. 

After discretization, we collect the set of free configurations $\mathcal C_{ij}$ on all grids, where a free configuration refers to the object being completely contained inside $P_i\cup P_j$. 
We then try to group free configurations on the same grid, forming of several disjoint subgraphs $\{G_{p,ij,n}\}_{n=1}^{n_{ij}}$, as shown in~\cref{fig:pipeline}. Within each subgraph, any two configurations can be interconnected via a collision-free path. 
This process is detailed in the supplementary materials\footnote{See: \href{https://sites.google.com/view/realm-rigidmilp}{https://sites.google.com/view/realm-rigidmilp}}.
Subsequently, the set of vertices $\mathcal V_{p,ij,n}\subset \mathcal C_{ij}$ of $n$-th subgraph $G_{p,ij,n}$ on $\SurfaceInter$ becomes a node $v_{ij, n}=\mathcal V_{p,ij,n}\in\DenseVert$ to serve as potential waypoints in detailed path planning. 
\setlength{\textfloatsep}{2pt}{
\begin{algorithm}[b]
    \caption{Construct a graph of convex polytopes $\DenseGraph$}\label{alg:dense-graph}
    \begin{algorithmic}[1]
        \REQUIRE list of obstacles $\ObstacleSet$, graph $\CoarseGraph$, number of intermediate waypoints $N$.
        \STATE $\DenseVert\gets\emptyset$, $\DenseEdge\gets\emptyset$
        \FOR{$(P_i, P_j)\in \CoarseEdge$}
            \STATE $v_{ij, 1},\cdots,v_{ij, n_{ij}}\gets\textsc{Sample\&Group}(\partial P_{ij}, \ObstacleSet)$
            \STATE $\DenseVert\gets\DenseVert\,\cup \{v_{ij, 1},\cdots,v_{ij, n_{ij}}\}$
        \ENDFOR
        \STATE $\text{CandidateEdgeSet}=\{(v_{ij,n_1},v_{ik,n_2})|v_{ij, n_1}, v_{ik,n_2}\in\DenseVert, v_{ij, n_1}\neq v_{ik,n_2}\}$
        \FOR{$(v_{ij,n_1},v_{ik,n_2})\in\text{CandidateEdgeSet}$}
            \IF{$\textsc{VerifyTraversal}(v_{ij, n_1},v_{ik, n_2}, \CoarseGraph, N$)}
                \STATE $\DenseEdge \gets \DenseEdge\,\cup \{(v_{ij, n_1},v_{ik, n_2})\}$
            \ENDIF
        \ENDFOR
        \RETURN $\DenseVert,\DenseEdge$
    \end{algorithmic}
\end{algorithm}}

\noindent\textbf{Edges $\DenseEdge$.} 
We assess all possible connections among $\DenseVert$ to establish $\DenseEdge$. For any two vertices $\DVertU$ and $\DVertW$, an edge is considered if the proposed traversal between vertices is verified as feasible by $\textsc{VerifyTraversal}$. 
The edge $(\DVertU,\DVertW)$ represents the feasibility of moving from polytope $P_j$ to $P_k$ via $P_i$.

\subsection{Online Query}\label{sec:inference}
Graph $\DenseGraph$ serves as an offline pre-computed roadmap, enabling rapid online path planning from a start configuration $q_{\text{start}}$ to an end configuration $q_{\text{end}}$.
We introduce new vertices $v_{\text{start}}=\{q_{\text{start}}\}$ and $v_{\text{end}}=\{q_{\text{end}}\}$ into $\DenseGraph$, connecting them using the same \textsc{VerifyTraversal} subroutine.
The resulting path consists of alternating segments: inter-vertex motions and intra-vertex motions. The path query first tries to extract a sequence of vertices from $v_{\text{start}}$ to $v_{\text{end}}$ on $\DenseGraph$. Inter-vertex motions are retrieved from the connecting edges. For each vertex $v_{ij,n}$ to be visited, the intra-vertex motions are obtained from the corresponding subgraph $G_{p,ij,n}$, to connect entry and exit configurations. All motion retrieval is performed online using Dijkstra's algorithm.

\vspace{-6pt}
\section{Verifying Path Segments}\label{sec:non-point}
A significant challenge arises in verifying the motion between waypoints due to the complex geometry of the objects. Previously, this is commonly addressed by bloating the obstacles to account for the size of objects, thereby transforming the problem into planning within this bloated environment as with a point object~\cite{demyen2006efficientTriangulation}.
However, this method can be overly restrictive, especially for objects whose shapes deviate significantly from circular or spherical forms. 

In this section, we will tackle the challenge using MILP to verify path segments for polytope objects.
We assume the rigid-body object $\Object$ to be simply connected (no holes) and its geometric center to lie inside the $\Object$.
We denote the vertices of $\Object$ are $\{v_{O,i}\}$.

%%%%%%%%%%%%%%%%%%%%%%%%%%%%%%%%%%%%%%%%%%%%%%%%%%%%%%%%%%%%%%%%%%%%%%%%%%%
\subsection{MILP Formulation}
As discussed in~\cref{sec:dense_graph}, the traversal between vertices $(\DVertU, \DVertW)$ in $\DenseGraph$ is not straightforward.
Therefore, we formulate the subroutine \textsc{VerifyTraversal} as an MILP, which is to verify whether there exists a piece-wise linear path from a configuration in $\DVertU$ to another one in $\DVertW$ with $N$ intermediate waypoints.

For the sake of simplicity, we'll denote $\DVertU$ as $u$ and $\DVertW$ as $w$ in the discussion of MILP, while denoting the set of polytopes $\{P_i, P_j, P_k\}\subset\CoarseEdge$ as $\mathcal P_c$. The list of free configurations in $u$ or $w$ are denoted as $(q_{u,1},\cdots,q_{u,n_u})$ and $(q_{w,1},\cdots,q_{w,n_w})$.
Let $q_t$ be the waypoints in the path, including the start and end configurations, indexed by $t\in\{0,\cdots,N+1\}$. Each waypoint $q_t$ consists of a translational part $p_t$ and a rotational part $R_t$. 
The translational part in a 2D scenario is given by $p_t = (x_t, y_t)$, and in 3D, it extends to $p_t = (x_t, y_t, z_t)$. 

\noindent\textbf{Model.} The objective is to minimize the total 1-norm of the translational displacement along the path. The entire MILP model is formulated in~\cref{eq:MILP-all}:
\begin{subequations}\label{eq:MILP-all}
\begin{flalign}
    \min_{\{p_t\}_t, \{R_t\}_t, \mathcal B} \; &\sum_{t=0}^{N}\|p_{t+1}-p_{t}\|_1, \label{eq:MILP-obj}\\
    \mathrm{s.t.} \; &~\cref{eq:rotation-idx,,eq:MILP-t-or-R,,eq:MILP-startend,,eq:MILP-collision}. \label{eq:all-constraints}
\end{flalign}
\end{subequations}
Here $\mathcal B$ is the set of all binary variables we'll introduce. The objective in~\cref{eq:MILP-obj} can be easily transformed to standard linear expressions with additional variables, which we will not detail here.

\noindent\textbf{Basic constraints.} The rotational part $R_t$ is encoded as a one-hot vector $\beta_{R,t}$ using binary variables. This is achieved through the constraints
\begin{subequations}\label{eq:rotation-idx}
\begin{flalign}
    && &R_t=\sum_{i=1}^{n_R}\beta_{R,t,i}R_i, & \forall t=0,\cdots,N+1 &&\\
    && &1 = \beta_{R,t}^T\mathbf 1, & \forall t=0,\cdots,N+1 &&
\end{flalign}
\end{subequations}
Each element in $R$ is a possible rotation matrix, as specified in ~\cref{sec:dense-node-full}. The index where $b_{R,t}=1$ indicates the selection of the corresponding rotation at waypoint $t$.  

Between two consecutive waypoints, we enforce the object to either translate or rotate, a decision encoded by binary variable $\beta_{\text{or},t}$ in~\cref{{eq:MILP-t-or-R}}. This constraint addresses the challenges associated with encoding collision-free conditions in MILPs. By restricting the motion to either translation or rotation at each step, this simplified model becomes computationally tractable. Specifically, in 3D scenarios, we further limit the object to translational motion only, equivalent to setting $\beta_{\text{or},t}=1$. 
\begin{subequations}\label{eq:MILP-t-or-R}
\begin{align}
    &\|p_{t+1}-p_{t}\|_1\le M\cdot \beta_{\text{or},t}, &&\forall t=0,\cdots,N\\
    &\|\beta_{R,t+1}-\beta_{R,t}\|_1\le M (1-\beta_{\text{or},t}), &&\forall t=0,\cdots,N\\
    &\beta_{\text{or},t}=1,\ (\text{only present in 3D}) &&\forall t=0,\cdots,N
\end{align}
\end{subequations}

The constraints for start and end configuration are that the first and last waypoints are in $u$ and $w$ correspondingly. We introduce two additional one-hot vectors $\beta_{\text{start}}$ and $\beta_{\text{end}}$, which are used to select specific configurations from the sets $u$ and $w$, shown in~\cref{eq:MILP-startend}.
\begin{subequations}\label{eq:MILP-startend}
\begin{align}
    & p_{0}=\sum_{i=1}^{n_u}\beta_{\text{start},i}p_{u,i},\ \ 
    && p_{N+1}=\sum_{i=1}^{n_w}\beta_{\text{end},i}p_{w,i},\\
    & R_{0}=\sum_{i=1}^{n_u}\beta_{\text{start},i}R_{u,i},\ \ 
    && R_{N+1}=\sum_{i=1}^{n_w}\beta_{\text{end},i}R_{w,i},\\
    & \sum_{i=1}^{n_u}\beta_{\text{start},i}=1,\ \ 
    && \sum_{i=1}^{n_w}\beta_{\text{end},i}=1
\end{align}
\end{subequations}

%%%%%%%%%%%%%%%%%%%%%%%%%%%%%%%%%%%%%%%%%%%%%%%%%%%%%%%%%%%%%%%%%%%%%%%%%%%
\subsection{Collision-Avoidance Constraint}
The remaining constraint in the MILP is collision avoidance. While the start ($q_0$) and end ($q_{N+1}$) configurations are already checked to be collision-free, ensuring continuous collision avoidance throughout the motion path is critical. 
Here, the collision avoidance constraint is encoded by enforcing that the each reachable set between two consecutive waypoints $(q_t,q_{t+1})$ ($t=0,\cdots,N$) be completely contained within the union of all relevant collision-free polytopes, $\bigcup \mathcal{P}_c$. Specifically, this is achieved by checking all the boundaries of the reachable sets being inside $\bigcup \mathcal{P}_c$.

%%%%%%%%%%%%%%%%%%%%%%%%%%%%%%%
\begin{figure}[!htb]
    \centering
    \includegraphics[width=0.9\linewidth]{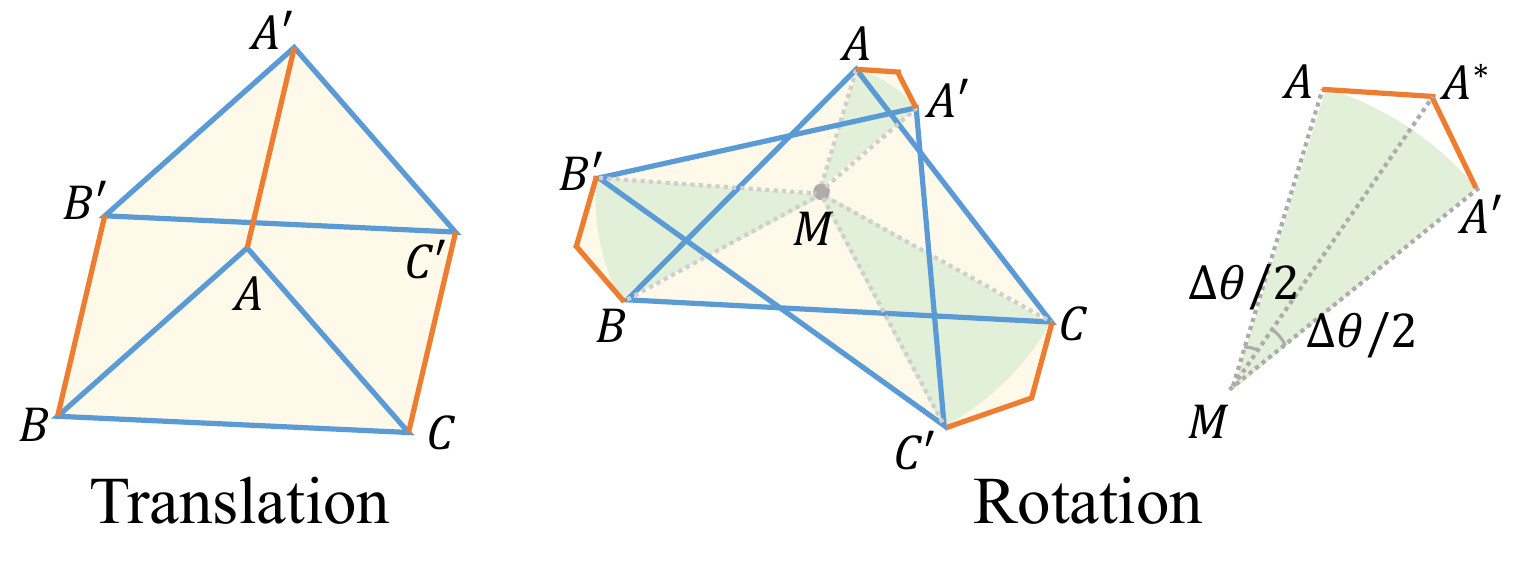}
    \caption{The figure shows the reachable set of 2D objects for translation and rotation scenarios, respectively. On the right, we illustrate the approximation of the green sectors to yellow polytopes.}
    \label{fig:reachable}
\end{figure}
\subsubsection{Compute boundary of reachable set}\label{subsubsec:reachable}
In 2D, the reachable set is computed as~\cref{fig:reachable}. For translation, the boundary contains the blue edges of $\Object$ at configurations $q_t$ and $q_{t+1}$, and orange lines connecting a same vertex at two waypoints. For rotation, the difference is that the lines connecting a same vertex at two waypoints are arcs. While the arcs can not be expressed in linear expression, we overapproximate the arc with two line segments, with $\|MA^*\|\cos(\Delta \theta_{\text{max}}/2)=\|MA\|$. $\Delta\theta_{\text{max}}$ is the maximum rotation angle allowed in one step, selected as $\pi/3$.
In 3D, only translation is allowed inside MILP. The rotation is only allowed inside each $V_d$ vertices, which is not solved by MILP. The reachable set is the union of the polytopes swept by each face of $\Object$, which is also a polytope. Therefore, to check the collisions of the faces of the reachable set, we only need to check the superset - the union of the triangle mesh faces of $\Object$ at $q_t$ and $q_{t+1}$ and the parallelogram swept by each edge of $\Object$.

%%%%%%%%%%%%%%%%%%%%%%%%%%%%%%%
\subsubsection{Collision detection between surface and polytope}
Compared with existing models of collision avoidance in~\cref{fig:collision-avoidance}, we propose a new encoding based on the following corollaries. We just need to check each edge on the boundary of the reachable set to satisfy~\cref{thm:line}. 
\begin{figure}[!htb]
    \centering
    \includegraphics[width=0.9\linewidth]{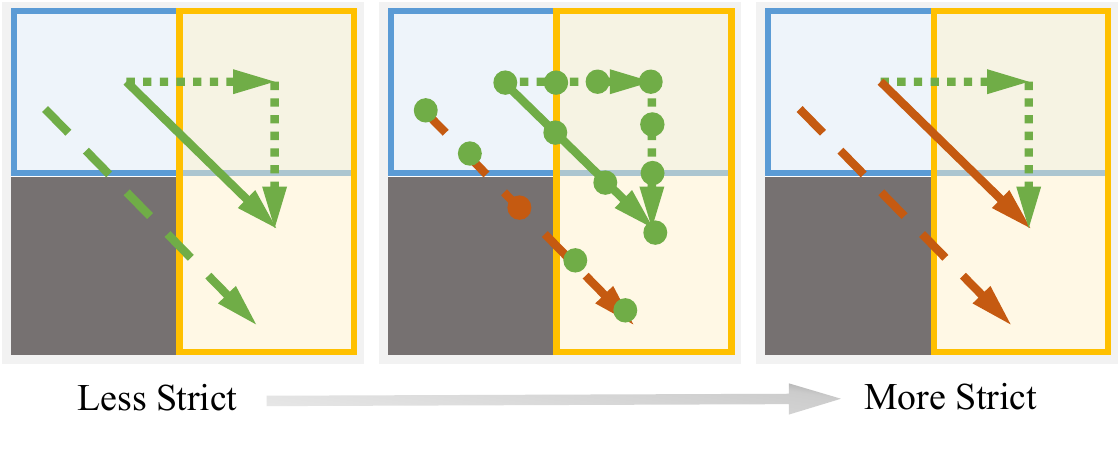}
    \caption{Models of collision avoidance around obstacles (gray). The free space is decomposed into convex polytopes (blue and yellow). Green lines indicate predicted collision-free paths and red lines indicate detected collisions.
    \textit{Left:} Less strict encoding that only checks endpoints of subdivided line segments. The dashed red line demonstrates that endpoint-only checking is insufficient as it cuts through an obstacle.
    \textit{Right}~\cite{8613017}: This overly restrictive model allows only paths like the dotted line where both endpoints lie within the same free region, rejecting the simpler, viable solid line.
    \textit{Middle (\textbf{ours}):} Strict encoding based on~\cref{thm:line}. This approach prevents the dashed line from cutting through obstacles by incorporating additional constraints beyond endpoint checking, while including more paths between regions, like the solid line.
    }
    \label{fig:collision-avoidance}
\end{figure}
\begin{proposition}[Line segment inside two convex polytopes] \label{thm:line}
    Given two convex polytopes $P_i=\{x|A_ix\le b_i\}$, $i=1,2$, and a line segment $l$ with two endpoints $(x_a,x_b)$, $l$ is contained inside the union of $P_1$ and $P_2$ if one of following conditions holds:
    \begin{enumerate}
        \item $x_a$ and $x_b$ are within a same polytope $P_i$,
        \item $x_a$ and $x_b$ are within different polytopes, but there exists a point $x$ satisfying $x\in (P_1\cap P_2)$.
    \end{enumerate}
    \end{proposition}\begin{proof}
    For condition (1), the convexity of $P_i$ ensures that all points on $l$, being a line segment between $x_a$ and $x_b$, are also contained within $P_i$. If condition (2) holds, line segments $(x_a,x)$ and $(x,x_b)$ both satisfy condition (1). Hence, the entire line segment $l$ remains within $P_1 \bigcup P_2$.
\end{proof}

We further extend the~\cref{thm:line} from checking line segments to triangles and convex quadrilaterals. By enforcing all edges of these shapes satisfying~\cref{thm:line}, we prove that the shapes are contained within the union of two given convex polytopes. The corollaries and proofs are provided in the supplementary materials
\footnote{See: \href{https://sites.google.com/view/realm-rigidmilp}{https://sites.google.com/view/realm-rigidmilp}}. 
So we just need to check the line segments within the sets $\mathcal S_\text{line}$ based on~\cref{thm:line}: \begin{enumerate}
    \item the line segments connecting the object's geometric center and vertices at waypoints $q_t$ and $q_{t+1}$
    \item the edges of $\Object$ at waypoints $q_t$ and $q_{t+1}$
    \item the line segments connecting a same vertex of $\Object$ between waypoints $q_t$ and $q_{t+1}$.
\end{enumerate}
Accurately verifying condition 2) of ~\cref{thm:line} involves a product of two sets of variables - the positions of $x_a,x_b$ and the proportion $(x,x_a)$ occupies. This product cannot be encoded into a mixed integer linear programming. Instead, we divide the line segments into $10$ equal parts and only check the $11$ endpoints in practice. So the MILP constraints can be written as, $\forall e_s=(p_{s,\text{start}},p_{s,\text{end}})\in \mathcal S_\text{line}$:
\begin{subequations}\label{eq:MILP-collision}
\begin{align}
    \begin{split}
        & A_i (\eta p_{s,\text{start}}+(1-\eta)p_{s,\text{end}})\le b_i + M(1-\beta_{s,i,\eta}), \\
        & \quad \forall P_i\in\mathcal P_c,\forall \eta\in \{0, 0.1,\cdots,1\}
    \end{split}\label{eq:AllInPoly1}\\
    & \sum_{i}\beta_{s,i,\eta}\ge 1,\quad \forall \eta\in \{0, 0.1,\cdots,1\}\label{eq:AllInPoly2}\\
    & \beta_{s,\text{cond1},i}\le \beta_{s,i,0},\quad \beta_{s,\text{cond1},i}\le \beta_{s,i,1} \label{eq:cond1}\\
    \begin{split}
        & \sum_\eta (\beta_{s,i,\eta} + \beta_{s,j,\eta}-1)\ge 1- M(1-\beta_{s,\text{cond2},ij}), \\
        & \quad \forall P_i,P_j\in\mathcal P_c,\  P_i\neq P_j
    \end{split}\label{eq:cond2}\\
    & \sum_i \beta_{s,\text{cond1},i} + \sum_{(i,j)}\beta_{s,\text{cond2},ij}\ge 1 \label{eq:cond1or2}
\end{align}
\end{subequations}
where $\beta_{s,i,\eta}$, $\beta_{s,\text{cond1},i}$, $\beta_{s,\text{cond2},ij}$ are binary variables for line segment $e_s$. $\beta_{s,i,\eta}=1$ indicates that interpolated points are inside polytope $P_i$. Similarly, $\beta_{s,\text{cond1},i}$ or $\beta_{s,\text{cond2},ij}$ being $1$ means the corresponding condition is satisfied.
~\cref{eq:AllInPoly1,,eq:AllInPoly2} encodes the preconditions, while~\cref{eq:cond1,,eq:cond2} encodes the two conditions in~\cref{thm:line}, respectively. 

%%%%%%%%%%%%%%%%%%%%%%%%%%%%%%%%%%%%%%%%%%%%%%%%%%%%%%%%%%%%%%%%%%%%%%%%%%%
\subsection{Discussion}~\label{subsec:limitation}
While our approach decomposes the original planning problem and solves each subproblem optimally, it does not guarantee global optimality for the entire path. However, our algorithm is capable of finding multiple solutions or modalities, which can then serve as initial solutions for further refinements. 
It is important to acknowledge that in the 3D case, our method currently only allows translation within the MILP formulation to simplify the computation of the reachable set. We recognize that this simplification may exclude some potential solutions. To address this limitation, future work will focus on extending the approximation techniques used in the 2D case to the 3D domain.

%%%%%%%%%%%%%%%%%%%%%%%%%%%%%%%%%%%%%%%%%%%%%%%%%%%%%%%%%%%%%%%%%%%%%%%%%%%
\begin{figure*}[t]
    \input{figs-final/main-env_vis/main-planning_result}
\end{figure*}

% \vspace{-6pt}
\section{Experiments}\label{sec:experiment}
We empirically validate our method in this section. 
We initially set the number of intermediate waypoints, $N$, to $0$. If no solution can be found with $N=0$, we increase $N$ to $1$. Additionally, rotations are discretized into $12$ distinct rotations for 2D environments and $24$ for 3D environments, respectively.
All experiments were launched on a server with 1 AMD Ryzen Threadripper 3990X 64-Core Processor. We adopt Gurobi 10.0.0~\cite{gurobi} as the MILP solver and handle the graphs $\CoarseGraph,\DenseGraph$ with NetworkX~\cite{networkx}.

\noindent\textbf{Benchmarks.} We collect $8$ benchmark environments, visualized in~\cref{fig:exp-env}. 
We consider two types of objects: convex and non-convex. In 2D environments, the convex object is selected as a stick of length $1.2$ and width $0.1$, while the non-convex object is an L-shape with longer side $1.2$, shorter side $0.8$ and width $0.1$. In 3D environments, the convex object is a pad of size $1.0\times 0.8\times0.1$, and the non-convex object is L-shaped, composed of two pads of size $1.0\times 0.8\times0.1$ and $1.0\times 0.1\times0.4$.

\noindent\textbf{Baselines.} We compare our method with the following multi-query motion planning algorithms, PRM~\cite{kavraki1996prm}, WCO~\cite{kurniawati2008workspace} and GCS~\cite{marcucci2023gcs}. 
PRM and WCO, sampling-based methods, are evaluated across $5$ trials for each problem set with an offline phase of 15 seconds allocated to develop a roadmap before each trial. 
GCS constructs a graph of convex set in configuration space~\cite{werner2023irisclique}, followed by optimizing for an optimal path with piece-wise linear curves~\cite{marcucci2023gcs}. 
To ensure a fair comparison in generating an IRIS cover for both GCS and our method, though their application in different spaces, we select the same set of hyperparameters. Specifically, we set $n_v=512$, $\alpha=0.95$ and select $n_s=5$ for our method.

    \begin{table}[h]
        \centering
        \caption{Comparison of our method and baselines on benchmarks.}
        \label{tab:planning-general}
        \begin{subtable}{0.5\textwidth}
  \centering
  \caption{Results for convex objects (A stick in 2D scenarios and a pad for 3D).}
  \setlength{\tabcolsep}{3pt}
  \begin{tabular}{ccccccccc}
    \toprule
    &\multicolumn{4}{c}{Online Time $\downarrow$ (ms)}    &  \multicolumn{4}{c}{Waypoint Number $\downarrow$}    \\
    \cmidrule(r){2-5}              \cmidrule(r){6-9}
    scenario  & ours   & GCS & PRM & WCO & ours  & GCS & PRM & WCO \\
    \midrule
    2d-corner   & $\mathbf{2.8}$ & $607.8$ & $105.3$ & $115.3$
                & $8$ & $\mathbf{6}$ & $20$ & $17$ \\
    2d-bugtrap     & $\mathbf{15.1}$ & INF & $116.4$ & $23.8$
                & $\mathbf{21}$ & INF & $40$ & $23$\\
    2d-maze        & $\mathbf{57.1}$ & $7.2\text{e}5$ & $111.9$ & $211.2$
                & $29$ & $\mathbf{21}$ & $42$ & $24$ \\
    SCOTS       & $\mathbf{55.0}$ & INF & $508.3$ & $360.4$
                & $85$   & INF & $97$ & $\mathbf{76}$ \\
    2d-peg      & $\mathbf{22.1}$ & INF & $122.6$ & $87.5$
                & $\mathbf{13}$ & INF & $31$ & $27$ \\
    \midrule
    3d-easy     & $\mathbf{63.8}$ & $1.2\text{e}5$ & $129.3$ & $244.4$
                & $10$ & $\mathbf{9}$ & $10$ & $24$ \\
    3d-narrow   & $\mathbf{212.4}$ & $2.0\text{e}5$ & $1045.0$ & $267.0$
                & $14$ & $32$ & $\mathbf{12}$ & $25$ \\
    3d-peg      & $\mathbf{48.1}$ & INF & $287.1$ & $162.1$
                & $20$ & INF & $\mathbf{6}$ & $21$ \\
    \bottomrule
  \end{tabular}
\end{subtable}\\
\vspace{3pt}
\begin{subtable}{0.5\textwidth}
  \centering
  \caption{Results for non-convex objects (L-shape for all scenarios).}
  \setlength{\tabcolsep}{3pt}
  \begin{tabular}{ccccccccc}
    \toprule
    &\multicolumn{4}{c}{Online Time $\downarrow$ (ms)}    &  \multicolumn{4}{c}{Waypoint Number $\downarrow$}    \\
    \cmidrule(r){2-5}              \cmidrule(r){6-9}
    scenario  & ours   & GCS & PRM & WCO & ours  & GCS & PRM & WCO \\
    \midrule
    2d-corner      & $\mathbf{2.8}$ & $4.1\text{e}4$ & $105.9$ & $258.3$
                & $\mathbf{9}$ & $11$ & $19$ & $21$ \\
    2d-bugtrap     & $\mathbf{11.5}$ & INF & $113.8$ & $25.4$
                & $\mathbf{18}$ & INF & $40$ & $29$ \\
    2d-maze        & $\mathbf{54.8}$ & $2.6\text{e}4$ & $110.4$ & $60.1$
                & $29$ & $\mathbf{25}$ & $43$ & $39$\\
    SCOTS       & $\mathbf{47.8}$ & INF & $3822.2$ & $473.8$
                & $97$ & INF & $98$ & $\mathbf{72}$\\
    2d-peg      & $\mathbf{12.8}$ & INF & $115.8$ & $106.6$
                & $\mathbf{13}$ & INF & $28$ & $17$\\
    \midrule
    3d-easy     & $\mathbf{66.7}$ & $2.1\text{e}5$ & $269.3$ & 106.3
                & $10$ & $\mathbf{5}$ & $7$ & $21$ \\
    3d-narrow   & $\mathbf{257.5}$ & INF & $593.7$ & $336.5$
                & $16$ & INF & $\mathbf{11}$ & $23$\\
    3d-peg      & $\mathbf{52.1}$ & INF & $824.7$ & $224.3$
                & $16$ & INF & $\mathbf{11}$ & $23$ \\
    \bottomrule
  \end{tabular}
\end{subtable}
    \end{table}
    
\subsection{Planning Results}
\noindent\textbf{Metrics.} In evaluating the performance of the methods, we assess two key metrics. First, we measure the \textbf{online time} required to compute a path when a new start and goal configuration pair is specified. 
Secondly, we evaluate the \textbf{number of waypoints} in the solution path. This metric serves as an indicator of the path's complexity and efficiency in navigating from start to goal. A lower number of waypoints generally suggests a smoother execution in practice.

\noindent\textbf{Result} 
We compare with baselines and show the results in~\cref{tab:planning-general}. We also visualized the solution paths in~\cref{fig:exp-env}. 

\paragraph{Online Time.} 
Our method consistently achieved the shortest online computation times, typically under 100 ms across most test scenarios. This efficiency contrasts with the GCS, which struggles to find feasible solutions in several environments. This is particularly evident in configurations involving narrow tunnels not covered by the C-IRIS strategy, which is explored further in \cref{sec:exp-gcs}.

The online time for our method involves retrieving the path from $\DenseGraph$ and connecting the start and end configurations to the graph. In practice, we set the number of intermediate waypoints $N=0$, allowing us to leverage matrix operations for quick verification instead of solving MILPs for faster online computation.
In contrast, the online stage for GCS involves solving an optimization problem. For PRM and WCO, it includes discrete collision checking to connect start and end configurations to a dense graph, much larger than ours.

\paragraph{Waypoint Comparison.} 
GCS, which optimizes for path optimality, generally produces smoother paths when it can find a solution, leading to fewer waypoints. 
In 2D environments, our method typically finds paths with fewer waypoints compared to sampling-based methods, indicating a more efficient pathfinding strategy. But the situation differs in 3D. This is because we over-restrict the free configurations on the boundaries in 3D scenarios, so the solution space is much smaller for our method and our method tends to generate more complex and twisted paths.

\paragraph{Impact of Object Shapes.} 
Our method allows the $\CoarseGraph$ to be reused for different object shapes. With the dimensions of the object being similar, our method and sampling-based methods are able to maintain a consistent online time for different objects. Conversely, GCS performance is significantly impacted by changes in object shape, altering the geometry of the free configuration space and affecting feasibility and computation times across various environments.

\begin{table*}[b]
    \centering
    \caption{Comparison of our method with sampling-based methods for an L-shaped object navigating a bugtrap environment~\cref{fig:env-bugtrap}. Online time is evaluated in both large and narrow environments. In large settings, the width mid-right corridor remains constant while the dimensions of the entire environment are doubled. In narrow settings, the corridor width is reduced to $60\%$ of its original size. In L \& N settings, the corridor width is reduced and the dimensions of the environment are doubled. The online time limit for PRM is $40$s.}
    \label{tab:prm}
    \begin{subtable}{1.\textwidth}
  \centering
  \begin{tabular}{ccccccccccc}
    \toprule
    Time (ms)     & ours   & PRM-$0.25$x & WCO-$0.25$x & PRM-$0.1$x & WCO-$0.1$x  & PRM-$0.05$x & WCO-$0.05$x  \\
    \midrule
     original       & $\mathbf{11.6}$ 
     & $113{\pm0.5}$ & $25.4{\pm8.1}$
     & $109{\pm1.4}$ & $30.7{\pm15.4}$
     & $127{\pm22.5}$ & $38.5{\pm19.1}$\\
     large          & $\mathbf{13.4}$ 
     & $521{\pm 823}$ & $39.6\pm19.3$
     & $4.2e3{\pm2.2e3}$ & $51.2\pm25.0$
     & $1.5e4{\pm1.2e4}$ & $69.3\pm33.7$\\
     narrow         & $\mathbf{13.0}$ 
     & $1.4e3{\pm 2.5e3}$ & $40.8\pm16.3$
     & $5.8e3{\pm 4.1e3}$ & $47.2\pm19.8$
     & $1.2e4{\pm 1.0e4}$ & $56.6\pm24.2$\\
     L \& N         & $\mathbf{17.4}$ 
     & $2.2e4{\pm1.5e4}$ & $279.3\pm208.4$
     & $3.7e4{\pm5.6e3}$ & $384.5\pm 291.5$
     & $3.9e4{\pm4.2e2}$ & $487.5\pm364.6$\\
    \bottomrule
  \end{tabular}
\end{subtable}
\end{table*}

\subsection{Comparison with Sampling-based Algorithms}\label{subsec:compare-prm}
Sampling-based methods, such as PRM and WCO, face challenges when scaling to larger environments. One critical hyperparameter affecting their performance is the resolution of collision checking. To assess the adaptability of these methods, we tested their efficiency with varying distances between checked states, specifically at $0.25$x, $0.1$x, and $0.05$x of the object's length, as detailed in~\cref{tab:prm}.

Our experiments revealed that both PRM and WCO struggle to maintain performance as the environment scales. In contrast, the offline computation time for our method consistently remains around 15 seconds across all test cases, attributable to the minimal changes in the structure of $\CoarseGraph$ and $\DenseGraph$ besides scaling. Similarly, the online execution time remains relatively consistent across scenarios. This is because the MILP approach in our method performs exact continuous collision checking with infinite resolution, enabling it to handle larger environments more effectively.
Moreover, we observed that WCO significantly outperforms PRM when scaling, demonstrating the benefit of utilizing workspace connectivity information. By leveraging this information, WCO can more efficiently explore the space and find feasible paths in larger and narrower environments compared to PRM.

\begin{table}[th]
    \centering
    \caption{Metrics of offline stages for our method and GCS.}
    \label{tab:gcs}
    \begin{subtable}{0.5\textwidth}
  \centering
  \begin{tabular}{ccccccc}
    \toprule
    &\multicolumn{3}{c}{Offline Time $\downarrow$ (s)}    
    &\multicolumn{2}{c}{\# IRIS regions $\downarrow$}    
    &\multicolumn{1}{c}{\# MILP $\downarrow$}\\
    \cmidrule(r){2-4}  \cmidrule(r){5-6}   \cmidrule(r){7-7}
    scenario  & $\CoarseGraph$ & $\DenseGraph$  & GCS & ours   & GCS  & ours \\
    \midrule
    corner      & $1.1$ & $0.5$ & $23.7$ 
                & $2$ & $24$ 
                & $6$ \\
    bugtrap     & $4.1$ & $10.7$ & $21.6$ 
                & $7$ & $42$ 
                & $79$ \\
    maze        & $6.1$ & $130.3$ & $251.4$ 
                & $19$ & $131$ 
                & $1624$ \\
    SCOTS       & $4.0$ & $45.7$ & $79.3$ 
                & $27$ & $78$ 
                & $355$ \\
    2d-peg      & $5.2$ & $18.8$ & $26.5$ 
                & $8$ & $59$ 
                & $170$ \\
    \midrule
    3d-easy     & $5.3$ & $15.8$ & $27.8$ 
                & $4$ & $55$ 
                & $42$ \\
    3d-narrow   & $7.7$ & $115.5$ & $96.9$ 
                & $8$ & $95$ 
                & $471$ \\
    3d-peg      & $16.4$ & $10.1$ & $166.9$ 
                & $4$ & $164$ 
                & $15$ \\
    \bottomrule
  \end{tabular}
\end{subtable}
\end{table}
\subsection{Comparison with GCS}~\label{sec:exp-gcs}
    The major difference between our method and GCS lies in space decomposition: our method is workspace-based while GCS works in C-space. 
    The dimensionality of C-space is usually higher than that of the workspace, accompanied by inherent geometrical and topological differences between these two spaces.
    In the demonstrated scenario~\cref{fig:vis-cfree}, where the objective is for a stick to navigate out of a trap through a narrow gap. The collision-free C-space consists of two large regions, connected by three thin tunnels. 
    Despite dense sampling revealing the tunnels, GCS, with $42$ polytopes, failed to detect any narrow passage, thus incapable of finding a feasible solution for the problem.
    On the contrary, our method identifies the critical pathways and achieves a near-complete coverage in the workspace.
    This distinction indicates a challenge in C-space representation, pathways that occupy a small volume in the workspace can correspond to an even smaller fraction of the free space in C-space, considering the increased dimensionality. This observation underscores the necessity of carefully designing a sampling strategy to ensure reliability.
    
    \begin{figure}[tb]
        \centering
        \includegraphics[width=.9\linewidth]{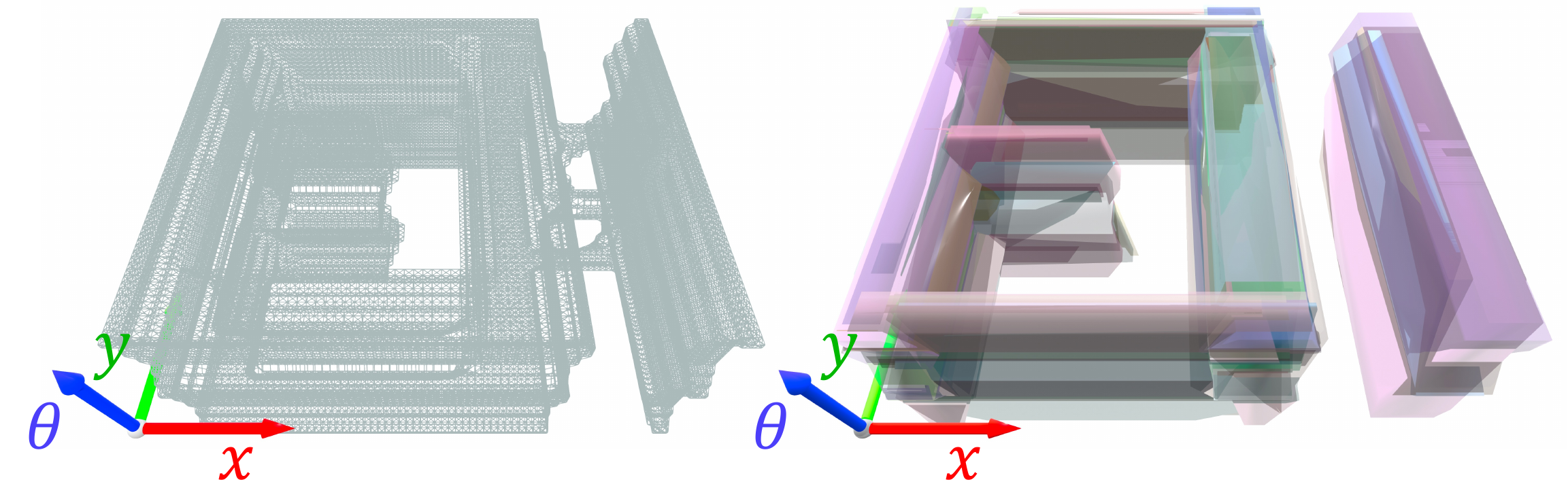}
        \caption{A 2D motion planning problem consists of a stick getting out of a trap and passing through a narrow gap, where our method is able to achieve almost $100\%$ coverage in workspace~\cref{fig:env-bugtrap}. 
        The C-space is a subset of $SE(2)$, with two translation axes and a periodic axis corresponding to the rotation. 
        We visualize the collision-free C-space by sampling (left) and the C-IRIS cover acquired from GCS (right).
        }
        \label{fig:vis-cfree}
    \end{figure}
    
    We also conduct a comprehensive comparison of time consumption and problem size for the offline stage. 
    The offline stage for our method can be roughly divided into two: the construction of $\CoarseGraph$ (\cref{sec:coarse_graph}) and $\DenseGraph$ (\cref{sec:dense_graph}). 
    % coarse-graph
    The problem size for $\CoarseGraph$ is the number of IRIS regions $(|\CoarseVert|)$ and its computation only needs to be conducted once across different objects in the same environment. But the graph in GCS needs to be recomputed when either the environment or the object changes.
    It can be observed that the number of IRIS regions we need to grow is significantly lower than GCS while remaining high feasibility.
    % dense-graph
    Moving to $\DenseGraph$, the problem size is determined by the number of MILP required. Despite our method necessitating a substantial number of MILPs, the MILPs are designed to be of small sizes, and batch processing is highly preferred. In the experiment, we assign $2$ threads for each MILP.
    The offline computation time for $\DenseGraph$ scales linearly to the number of MILP. Conversely for GCS, the number of variables is proportional to the number of IRIS regions, which makes the online time extremely long for GCS when the number of IRIS regions is large.
    
    % total time
    Across the majority of test cases, our method demonstrated reduced offline computation times while maintaining feasibility. 
    However, it's important to note the inherent advantage of GCS in its design for globally optimal solutions. While our method focuses on achieving faster computation and higher feasibility, this comes at the sacrifice of optimality.

\section{Conclusion} 
\label{sec:conclusion}
In this study, we introduced a path-planning approach for rigid body objects using MILP with better scalability and efficiency. 
By structuring our method into three distinct stages — workspace decomposition, path segment validation, and online query for rapid path retrieval, we have demonstrated improved planning efficiency in multi-query scenarios via comprehensive experiments against baselines.

% % \addtolength{\textheight}{-12cm}   % This command serves to balance the column lengths
% %                                   % on the last page of the document manually. It shortens
% %                                   % the textheight of the last page by a suitable amount.
% %                                   % This command does not take effect until the next page
% %                                   % so it should come on the page before the last. Make
% %                                   % sure that you do not shorten the textheight too much.

% %%%%%%%%%%%%%%%%%%%%%%%%%%%%%%%%%%%%%%%%%%%%%%%%%%%%%%%%%%%%%%%%%%%%%%%%%%%%%%%%

% \section*{ACKNOWLEDGMENT}

%%%%%%%%%%%%%%%%%%%%%%%%%%%%%%%%%%%%%%%%%%%%%%%%%%%%%%%%%%%%%%%%%%%%%%%%%%%%%%%%

\bibliographystyle{IEEEtran}
\bibliography{IEEEtranBST/IEEEabrv, references.bib}  % .bib

\begin{thebibliography}{10}
\providecommand{\url}[1]{#1}
\csname url@samestyle\endcsname
\providecommand{\newblock}{\relax}
\providecommand{\bibinfo}[2]{#2}
\providecommand{\BIBentrySTDinterwordspacing}{\spaceskip=0pt\relax}
\providecommand{\BIBentryALTinterwordstretchfactor}{4}
\providecommand{\BIBentryALTinterwordspacing}{\spaceskip=\fontdimen2\font plus
\BIBentryALTinterwordstretchfactor\fontdimen3\font minus \fontdimen4\font\relax}
\providecommand{\BIBforeignlanguage}[2]{{%
\expandafter\ifx\csname l@#1\endcsname\relax
\typeout{** WARNING: IEEEtran.bst: No hyphenation pattern has been}%
\typeout{** loaded for the language `#1'. Using the pattern for}%
\typeout{** the default language instead.}%
\else
\language=\csname l@#1\endcsname
\fi
#2}}
\providecommand{\BIBdecl}{\relax}
\BIBdecl

\bibitem{schwartz1983piano}
J.~T. Schwartz and M.~Sharir, ``\href{https://onlinelibrary.wiley.com/doi/pdf/10.1002/cpa.3160360305}{On the “piano movers'” problem I. The case of a two-dimensional rigid polygonal body moving amidst polygonal barriers},'' \emph{Communications on pure and applied mathematics}, vol.~36, no.~3, pp. 345--398, 1983.

\bibitem{kavraki1996prm}
L.~E. Kavraki, P.~Svestka, J.-C. Latombe, and M.~H. Overmars, ``\href{https://ieeexplore.ieee.org/document/508439/}{Probabilistic roadmaps for path planning in high-dimensional configuration spaces},'' \emph{IEEE transactions on Robotics and Automation}, vol.~12, no.~4, pp. 566--580, 1996.

\bibitem{tomlin2008tunnel_milp}
M.~Vitus, V.~Pradeep, G.~Hoffmann, S.~Waslander, and C.~Tomlin, ``\href{{https://arc.aiaa.org/doi/pdf/10.2514/6.2008-7132}}{Tunnel-milp: Path planning with sequential convex polytopes},'' in \emph{AIAA guidance, navigation and control conference and exhibit}, p. 7132.

\bibitem{dobson2014sparse}
A.~Dobson and K.~E. Bekris, ``Sparse roadmap spanners for asymptotically near-optimal motion planning,'' \emph{The International Journal of Robotics Research}, vol.~33, no.~1, pp. 18--47, 2014.

\bibitem{nister2023nonholonomic}
D.~Nister, J.~Soundararajan, Y.~Wang, and H.~Sane, ``\href{https://arxiv.org/abs/2306.01301}{Nonholonomic Motion Planning as Efficient as Piano Mover's},'' 2023.

\bibitem{8613017}
M.~d.~S. Arantes, C.~F.~M. Toledo, B.~C. Williams, and M.~Ono, ``\href{https://ieeexplore.ieee.org/document/8613017}{Collision-Free Encoding for Chance-Constrained Nonconvex Path Planning},'' \emph{IEEE Transactions on Robotics}, vol.~35, no.~2, pp. 433--448, 2019.

\bibitem{amice2022ciris}
A.~Amice, H.~Dai, P.~Werner, A.~Zhang, and R.~Tedrake, ``\href{https://arxiv.org/pdf/2205.03690}{Finding and optimizing certified, collision-free regions in configuration space for robot manipulators},'' in \emph{International Workshop on the Algorithmic Foundations of Robotics}.\hskip 1em plus 0.5em minus 0.4em\relax Springer, 2022, pp. 328--348.

\bibitem{kurniawati2008workspace}
H.~Kurniawati and D.~Hsu, ``Workspace-based connectivity oracle: An adaptive sampling strategy for prm planning,'' in \emph{Algorithmic Foundation of Robotics VII: Selected Contributions of the Seventh International Workshop on the Algorithmic Foundations of Robotics}.\hskip 1em plus 0.5em minus 0.4em\relax Springer, 2008, pp. 35--51.

\bibitem{5477241}
E.~Plaku, L.~E. Kavraki, and M.~Y. Vardi, ``\href{https://ieeexplore.ieee.org/document/5477241}{Motion Planning With Dynamics by a Synergistic Combination of Layers of Planning},'' \emph{IEEE Transactions on Robotics}, vol.~26, no.~3, pp. 469--482, 2010.

\bibitem{mcnaughton2011parallel}
M.~McNaughton, \emph{\href{https://www.ri.cmu.edu/pub_files/2011/7/mcnaughton-thesis.pdf}{Parallel algorithms for real-time motion planning}}.\hskip 1em plus 0.5em minus 0.4em\relax Carnegie Mellon University, 2011.

\bibitem{kelly2006toward}
A.~Kelly, A.~Stentz, O.~Amidi, M.~Bode, D.~Bradley, A.~Diaz-Calderon, M.~Happold, H.~Herman, R.~Mandelbaum, T.~Pilarski \emph{et~al.}, ``\href{https://journals.sagepub.com/doi/pdf/10.1177/0278364906065543}{Toward reliable off road autonomous vehicles operating in challenging environments},'' \emph{The International Journal of Robotics Research}, vol.~25, no. 5-6, pp. 449--483, 2006.

\bibitem{pivtoraiko2009differentially}
M.~Pivtoraiko, R.~A. Knepper, and A.~Kelly, ``\href{https://onlinelibrary.wiley.com/doi/abs/10.1002/rob.20285}{Differentially constrained mobile robot motion planning in state lattices},'' \emph{Journal of Field Robotics}, vol.~26, no.~3, pp. 308--333, 2009.

\bibitem{chew1987constraineddelaunay}
L.~P. Chew, ``\href{https://dl.acm.org/doi/pdf/10.1145/41958.41981}{Constrained delaunay triangulations},'' in \emph{Proceedings of the third annual symposium on Computational geometry}, 1987, pp. 215--222.

\bibitem{demyen2006efficientTriangulation}
D.~Demyen and M.~Buro, ``\href{https://cdn.aaai.org/AAAI/2006/AAAI06-148.pdf}{Efficient triangulation-based pathfinding},'' in \emph{AAAI}, vol.~6, 2006, pp. 942--947.

\bibitem{932817}
O.~Brock and L.~Kavraki, ``Decomposition-based motion planning: a framework for real-time motion planning in high-dimensional configuration spaces,'' in \emph{Proceedings 2001 ICRA. IEEE International Conference on Robotics and Automation}, vol.~2, 2001, pp. 1469--1474 vol.2.

\bibitem{deits2015iris}
R.~Deits and R.~Tedrake, ``\href{https://link.springer.com/chapter/10.1007/978-3-319-16595-0_7}{Computing large convex regions of obstacle-free space through semidefinite programming},'' in \emph{Algorithmic Foundations of Robotics XI: Selected Contributions of the Eleventh International Workshop on the Algorithmic Foundations of Robotics}.\hskip 1em plus 0.5em minus 0.4em\relax Springer, 2015, pp. 109--124.

\bibitem{marcucci2024spp-gcs}
T.~Marcucci, J.~Umenberger, P.~Parrilo, and R.~Tedrake, ``Shortest paths in graphs of convex sets,'' \emph{SIAM Journal on Optimization}, vol.~34, no.~1, pp. 507--532, 2024.

\bibitem{lavalle2001randomized}
S.~M. LaValle and J.~J. Kuffner~Jr, ``\href{https://journals.sagepub.com/doi/full/10.1177/02783640122067453}{Randomized kinodynamic planning},'' \emph{The international journal of robotics research}, vol.~20, no.~5, pp. 378--400, 2001.

\bibitem{quinlan1995real}
S.~Quinlan, \emph{\href{http://infolab.stanford.edu/pub/cstr/reports/cs/tr/95/1537/CS-TR-95-1537.pdf}{Real-time modification of collision-free paths}}.\hskip 1em plus 0.5em minus 0.4em\relax Stanford University, 1995.

\bibitem{schwarzer2005adaptive}
F.~Schwarzer, M.~Saha, and J.-C. Latombe, ``\href{https://ieeexplore.ieee.org/iel5/8860/30927/01435478.pdf}{Adaptive dynamic collision checking for single and multiple articulated robots in complex environments},'' \emph{IEEE Transactions on Robotics}, vol.~21, no.~3, pp. 338--353, 2005.

\bibitem{845313}
S.~Redon, A.~Kheddar, and S.~Coquillart, ``\href{https://ieeexplore.ieee.org/stamp/stamp.jsp?tp=&arnumber=845313}{An algebraic solution to the problem of collision detection for rigid polyhedral objects},'' in \emph{Proceedings 2000 ICRA. Millennium Conference. IEEE International Conference on Robotics and Automation. Symposia Proceedings (Cat. No.00CH37065)}, vol.~4, 2000, pp. 3733--3738 vol.4.

\bibitem{zhang2007continuous}
X.~Zhang, S.~Redon, M.~Lee, and Y.~J. Kim, ``\href{https://dl.acm.org/doi/pdf/10.1145/1276377.1276396}{Continuous collision detection for articulated models using taylor models and temporal culling},'' \emph{ACM Trans. Graph.}, vol.~26, no.~3, pp. 15--es, 2007.

\bibitem{werner2023irisclique}
P.~Werner, A.~Amice, T.~Marcucci, D.~Rus, and R.~Tedrake, ``\href{https://arxiv.org/pdf/2310.02875.pdf}{Approximating robot configuration spaces with few convex sets using clique covers of visibility graphs},'' \emph{arXiv preprint arXiv:2310.02875}, 2023.

\bibitem{2016scots}
M.~Rungger and M.~Zamani, ``\href{https://dl.acm.org/doi/10.1145/2883817.2883834}{SCOTS: A Tool for the Synthesis of Symbolic Controllers},'' in \emph{Proceedings of the 19th international conference on hybrid systems: Computation and control}, 2016, pp. 99--104.

\bibitem{gurobi}
{Gurobi Optimization, LLC}, ``\href{https://www.gurobi.com}{Gurobi Optimizer Reference Manual},'' 2023.

\bibitem{networkx}
A.~A. Hagberg, D.~A. Schult, and P.~J. Swart, ``\href{https://conference.scipy.org/proceedings/SciPy2008/paper_2/}{Exploring Network Structure, Dynamics, and Function using NetworkX},'' in \emph{Proceedings of the 7th Python in Science Conference}, G.~Varoquaux, T.~Vaught, and J.~Millman, Eds., Pasadena, CA USA, 2008, pp. 11 -- 15.

\bibitem{marcucci2023gcs}
T.~Marcucci, M.~Petersen, D.~von Wrangel, and R.~Tedrake, ``Motion planning around obstacles with convex optimization,'' \emph{Science robotics}, vol.~8, no.~84, p. eadf7843, 2023.

\end{thebibliography}

% \twocolumn
\clearpage
\appendix
\setcounter{proposition}{1}
\subsection{Constructing Vertices $\DenseVert$}~\label{sec:connectivity-patch}
In constructing the vertices $\DenseVert$, we group free configurations on the grid to form disjoint subgraphs $\{G_{p,ij,n}\}_{n=1}^{n_ij}$. The simplest approach is to connect free configurations which are adjacent on the discretization grid (for example in 2d, the discretization grid is a 2-dimensional plane of discretized translation position and discretized rotation angle). However, simply connecting adjacent free configurations may lead to collisions for non-point objects. To address this, we enlarge the polytope object $\Object$ during collision checking. This ensures that movements between adjacent configurations remain collision-free.

\noindent\textbf{2D Scenario.} For each discrete configuration on this grid, we conduct collision checks for a bloated object $O'$. Each edge on object is $\Object$ bloated by distance $\max(l_\lambda,l_\theta)$ where $l_\lambda$ is the maximum distance between two consecutive translation points and $l_\theta$ is the maximum distance a point on $\Object$ could possibly travel between consecutive rotations, calculated by $\max_i 2\|v_{O,i}\|_2\sin(\pi/n_R)$.

\noindent\textbf{3D Scenario.}  We over-approximate the object by a ball of radius $\max_i|v_{O,i}|_2$ for this collision checking. This ball accommodates arbitrary rotations for configurations that share the same translation. For configurations on the same facet that share the same rotation, connectivity does not require verification due to the convexity of $P_i$ and $P_j$. This is justified as follows: Object $\Object$ is divided into two parts by the facet of the reachable set, one part resides in $P_i$ and the other inside $P_j$. This division remains fixed for configurations on the same facet under the same rotation. Given that both $P_i$ and $P_j$ are convex, any swept region of each part of $\Object$ remains confined within its respective convex polytope.

\subsection{Proofs of Corollaries}
We will provide proofs of two corollaries presented in the paper here.
\setcounter{corollary}{0}
\begin{corollary}[Triangle] \label{thm:triangle}
    Given two convex polytopes $P_i=\{x|A_ix\le b_i\}$, $i=1,2$, and a triangle $T$, $T$ is contained inside the union of $P_1$ and $P_2$ if all edges of $T$ satisfy at least one of the conditions in Proposition 1.%~\cref{thm:line}.
\end{corollary}
\begin{figure}[h]
    \centering
    \includegraphics[width=1\linewidth]{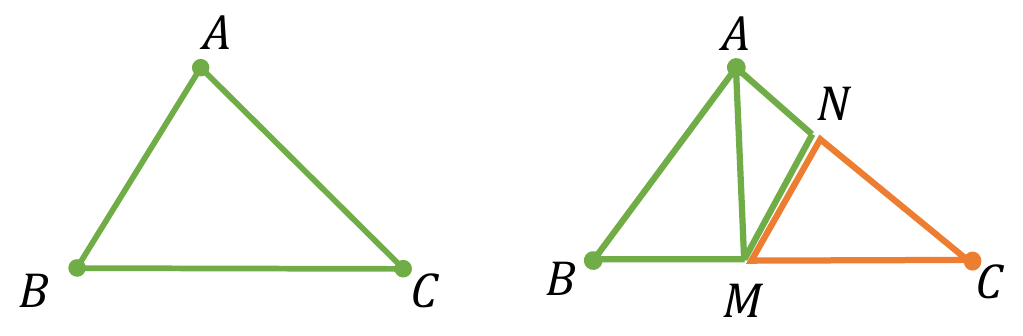}
    \caption{
    \textit{Left:} all the vertices are contained in the same green polytope.
    \textit{Right:} one vertex ($C$) is not contained in the same polytope as vertices $A,B$.
    }
    \label{fig:proof-triangle}
\end{figure}
\begin{proof}
We first consider the case when all the vertices of triangle $T$ are contained in exactly one of $P_i$, there are two possible cases as demonstrated in~\cref{fig:proof-triangle}: 
\begin{enumerate}
    \item All vertices of triangle $T$ are contained within the same convex polytope, here referred to as the green polytope. Since convex polytopes maintain the property that any point on the line segment between any two points within the polytope also lies within the polytope, it follows that the entire area of triangle $T$ is contained within the green polytope.
    \item One vertex $C$ is located in a different polytope from vertices $A,B$ (\textit{right}). Given that the edge $BC$ satisfies Proposition 1%~\cref{thm:line}
    , there exists a point $M$ on $BC$ such that $M \in P_1 \cap P_2$. Similarly, for edge $AC$, there exists a point $N$ on $AC$ that also belongs to $P_1 \cap P_2$. 
    The triangle $T$, i.e. $\triangle ABC$ can be divided into three triangles: $\triangle ABM,\triangle AMN, \triangle MNC$. Since each triangle's vertices are entirely contained within at least one polytope and each polytope is convex, each of these smaller triangles is contained within a polytope. Hence, the entire triangle $T$ is contained within the union of $P_1$ and $P_2$.
\end{enumerate}

When extending the conditions to allow vertices of triangle $T$ to be contained in more than one of $P_i$, it is equivalent to considering an enlargement of the polytopes $P_1$ and $P_2$. Therefore, the proposition that triangle $T$ is contained within the union of the polytopes still holds.
\end{proof}

\begin{corollary}[Convex quadrilateral] \label{thm:Parallelogram}
    Given two convex polytopes $P_i=\{x|A_ix\le b_i\}$, $i=1,2$, and a convex quadrilateral $P_0$, $P_0$ is contained inside the union of $P_1$ and $P_2$ if all edges of $P_0$ satisfy at least one of the conditions in Proposition 1.%~~\cref{thm:line}.
\end{corollary}
\begin{figure}[h]
    \centering
    \includegraphics[width=1\linewidth]{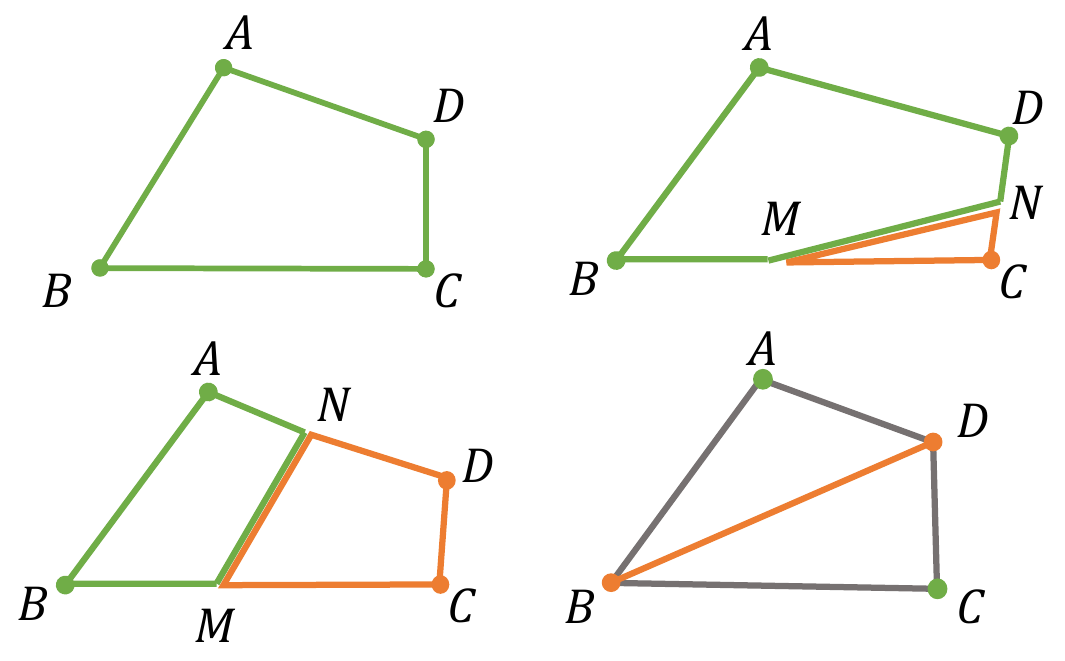}
    \caption{
    \textit{Upper left:} all the vertices are contained in the same green polytope.
    \textit{Upper right:} vertex $C$ is not contained in the same polytope as the other vertices ($A, B, D$)
    \textit{Lower left:} two adjacent vertices ($C,D$) are contained in a different polytope as $A,B$.
    \textit{Lower right:} vertices along the same diagonal ($AC$ and $BD$) are contained in a same polytope.
    }
    \label{fig:proof-quadrilateral}
\end{figure}
\begin{proof}
Similar to the proof of triangles, we first consider the case when all the vertices of convex quadrilateral $P_0$ are contained in exactly one of $P_i$, there are four possible cases as demonstrated in~\cref{fig:proof-quadrilateral}: 
\begin{enumerate}
    \item All vertices of convex quadrilateral $P_0$ are contained within the same convex polytope, here referred to as the green polytope. Since convex polytopes maintain the property that any point on the line segment between any two points within the polytope also lies within the polytope, it follows that the entire area of convex quadrilateral $P_0$is contained within the green polytope.
    \item One vertex $C$ is located in a different polytope from vertices $A,B,D$. Given that the edge $BC$ satisfies Proposition 1%~\cref{thm:line}
    , there exists a point $M$ on $BC$ such that $M \in P_1 \cap P_2$. Similarly, for edge $CD$, there exists a point $N$ on $CD$ that also belongs to $P_1 \cap P_2$. 
    The convex quadrilateral $P_0$, i.e. $\triangle ABC$ can be divided into two polytopes: polytope $ABMND$ and $\triangle MNC$. Since each small polytope's vertices are entirely contained within at least one $P_i$ and each $P_i$ is convex, each of these smaller polytopes is contained within a polytope. Hence, the entire convex quadrilateral $P_0$ is contained within the union of $P_1$ and $P_2$.
    \item Two adjacent vertices $C,D$ are contained in a different polytope as $A,B$. Similarly, there exists a point $M$ on $BC$ and a point $N$ on $AD$ such that $M,N \in P_1 \cap P_2$. The convex quadrilateral $P_0$, i.e. $\triangle ABC$ can be divided into two polytopes: quadrilateral $ABMN$ and quadrilateral $MNDC$. Since each small polytope's vertices are entirely contained within at least one $P_i$ and each $P_i$ is convex, each of these smaller polytopes is contained within a polytope. Hence, the entire convex quadrilateral $P_0$ is contained within the union of $P_1$ and $P_2$.
    \item Vertices along the same diagonal ($AC$ and $BD$) are contained in a same polytope. We divide the quadrilateral $ABCD$ into two triangles $\triangle ABD$ and $\triangle BCD$. The two vertices of edge $BD$ are contained in the same convex polytope, so both two small triangles satisfy~\cref{thm:triangle}. Hence, the entire convex quadrilateral $P_0$ is contained within the union of $P_1$ and $P_2$.
\end{enumerate}
\begin{figure*}[h]
    \centering
    \begin{subfigure}[b]{0.19\textwidth}
        \centering
        \includegraphics[height=9em]{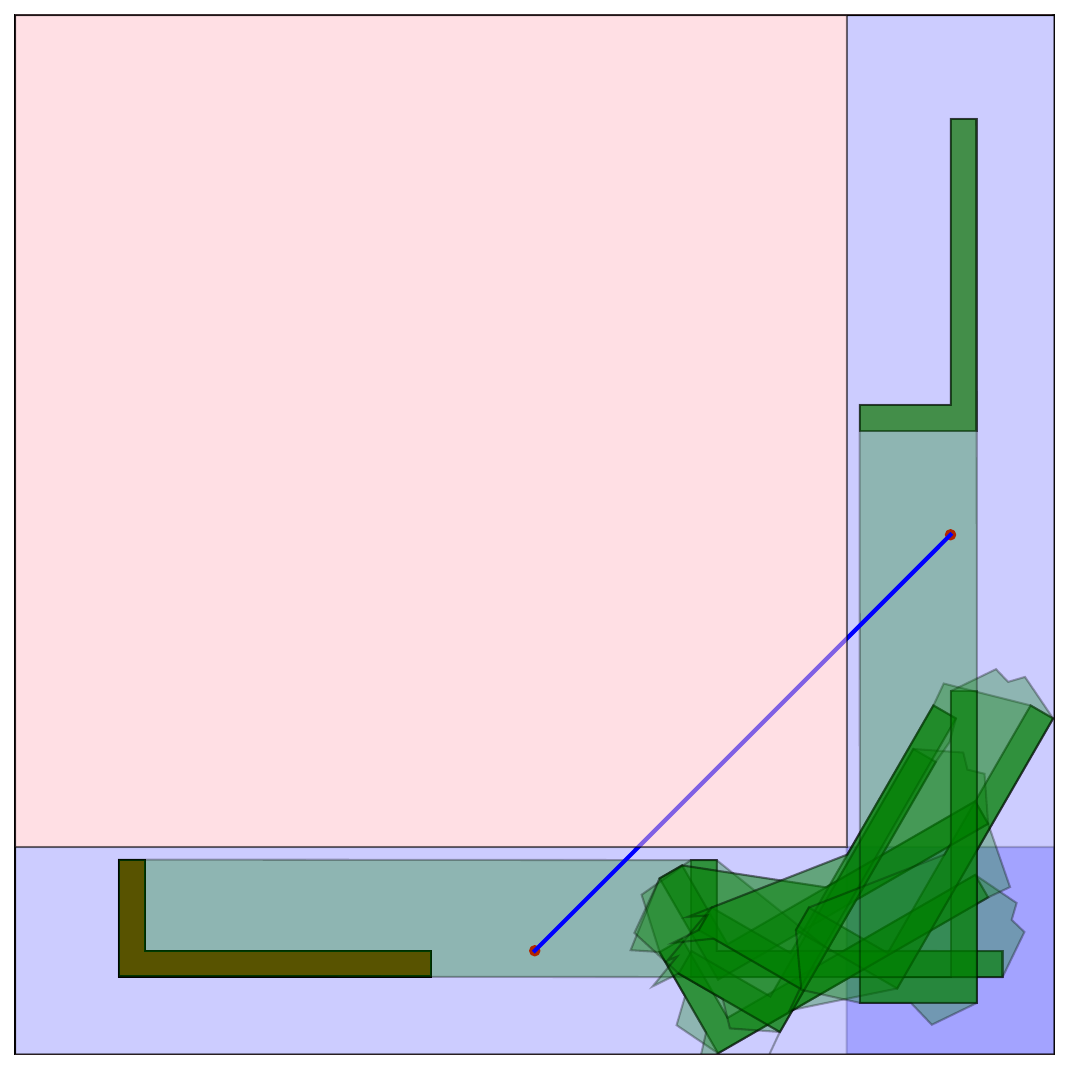}
        \caption{2d-Corner.}
    \end{subfigure}
    \hfill
    \begin{subfigure}[b]{0.25\textwidth}
        \centering
        \includegraphics[height=9em]{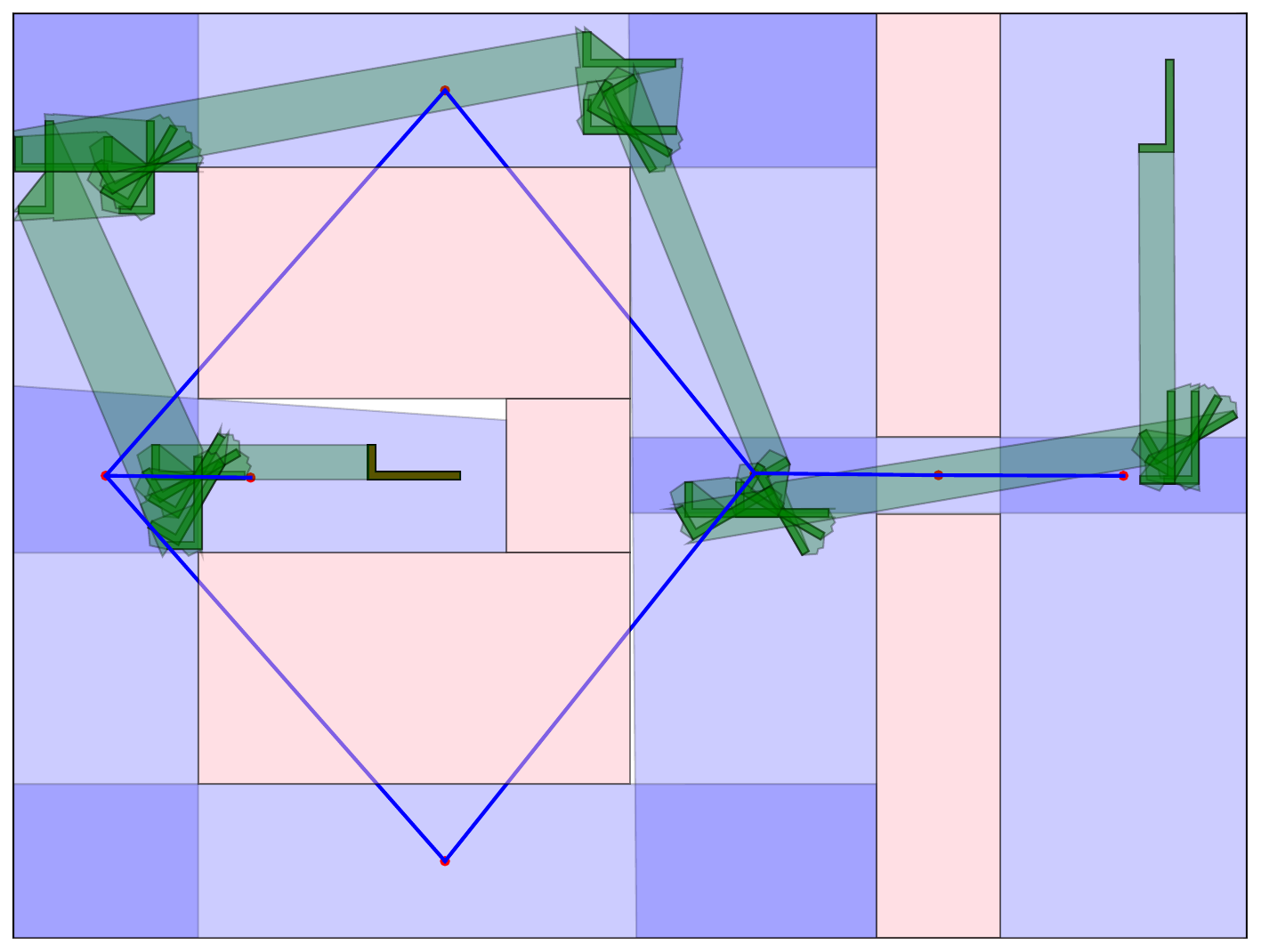}
        \caption{2d-bugtrap.}
    \end{subfigure}
    \hfill
    \begin{subfigure}[b]{0.25\textwidth}
        \centering
        \includegraphics[height=9em]{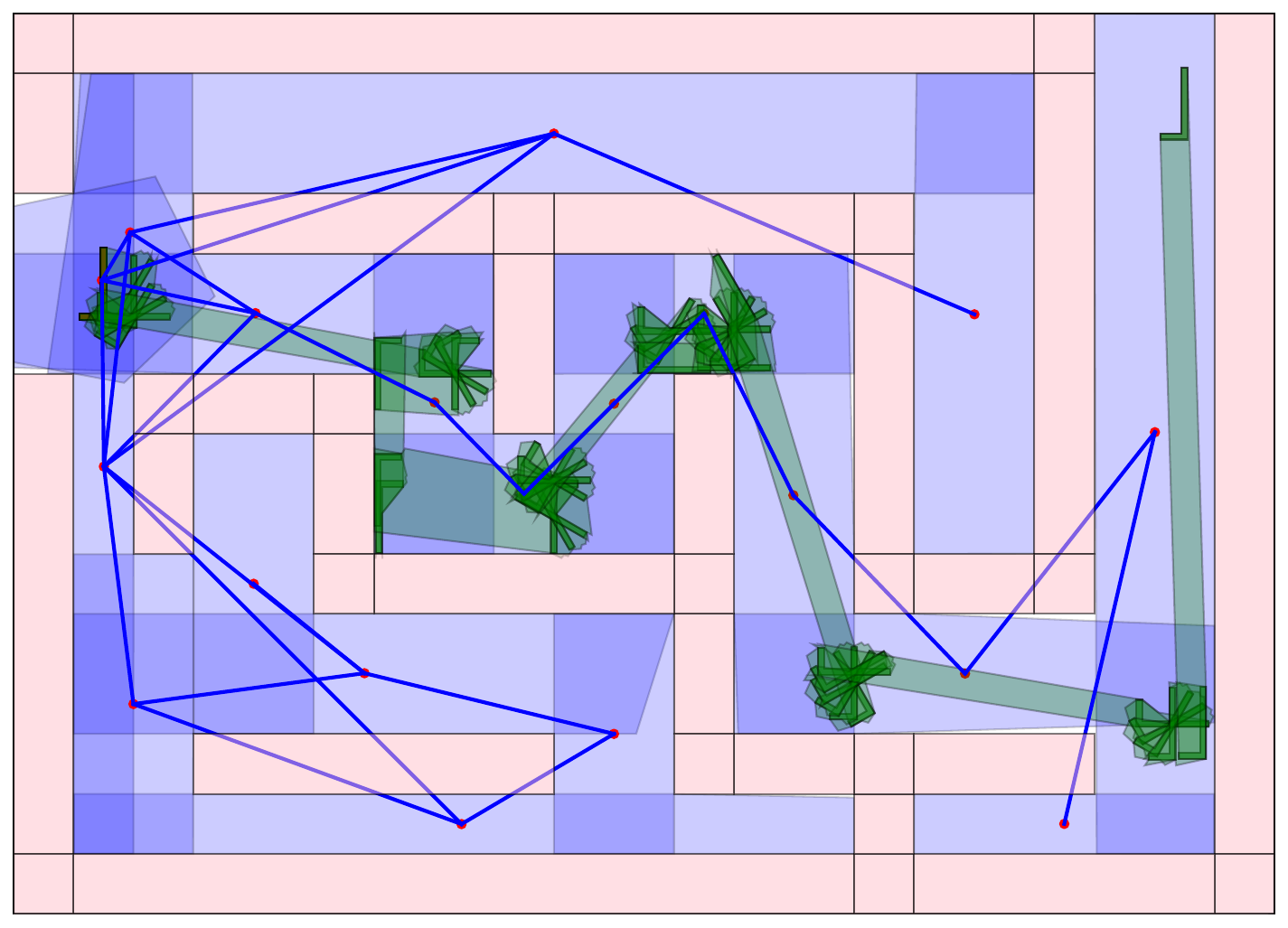}
        \caption{2d-maze.}
    \end{subfigure}
    \hfill
    \begin{subfigure}[b]{0.27\textwidth}
        \centering
        \includegraphics[height=9em]{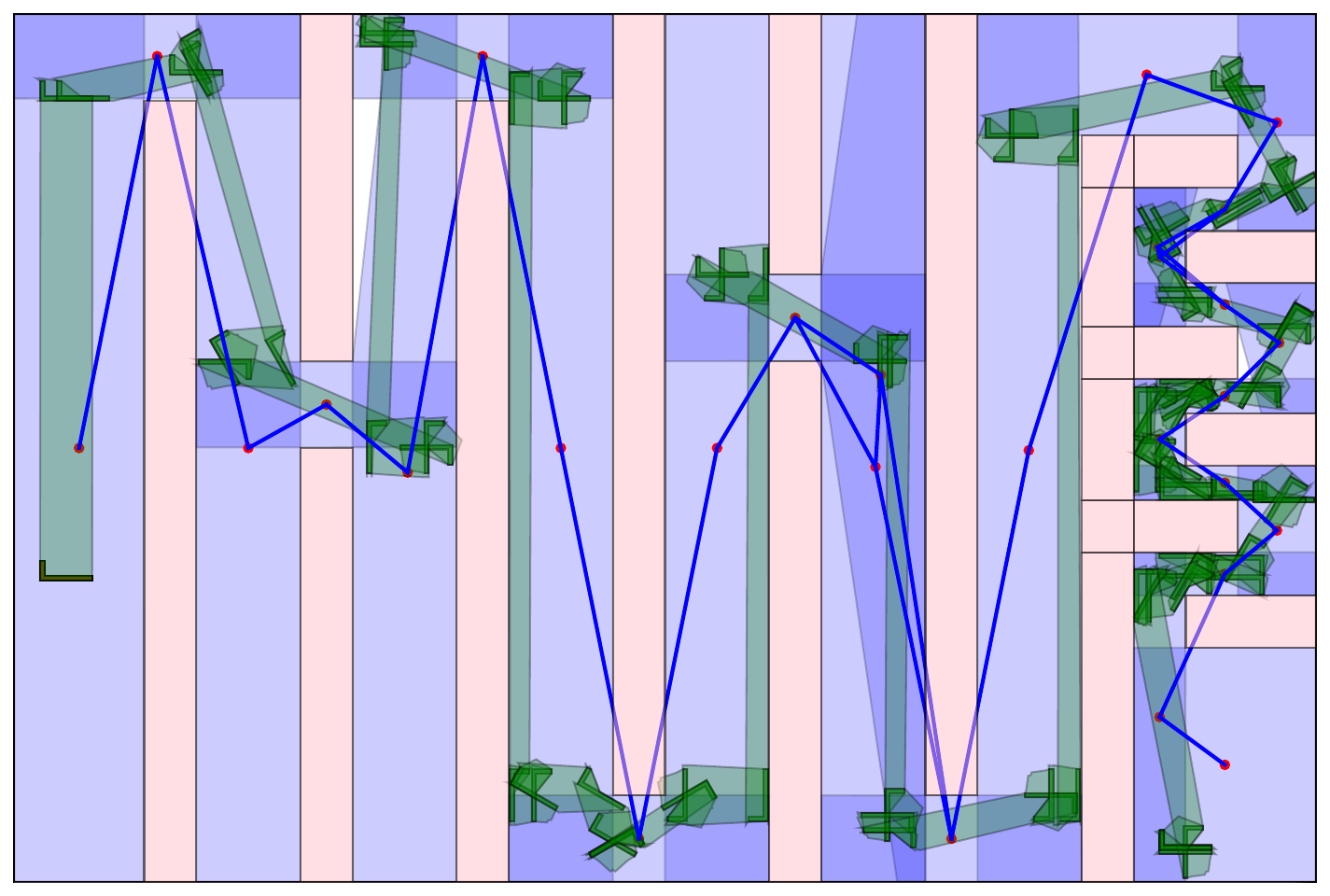}
        \caption{SCOTS.~\cite{2016scots}}
    \end{subfigure} \\ \vspace{.5em}
    \begin{subfigure}[b]{0.19\textwidth}
        \centering
        \includegraphics[height=9em]{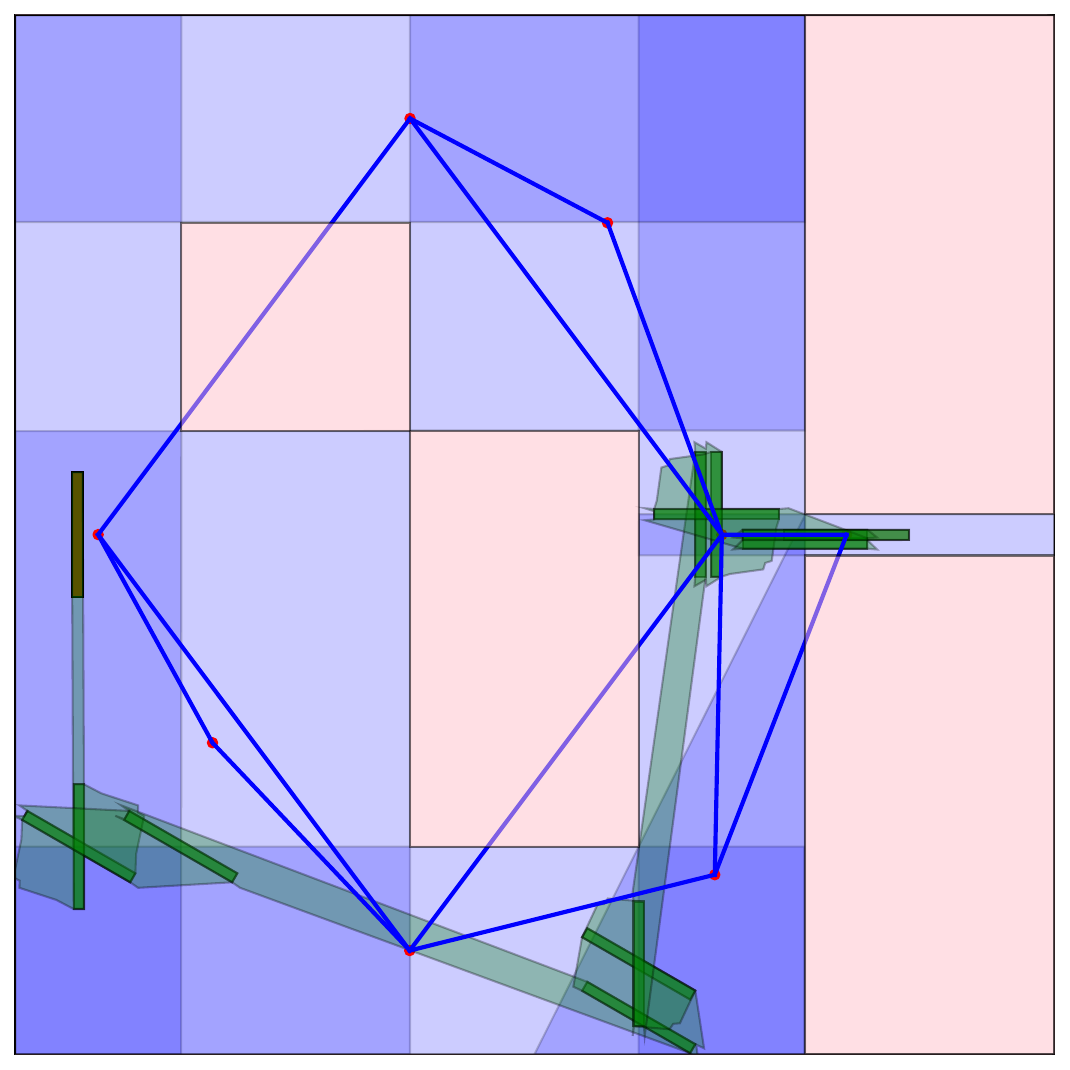}
        \caption{2d-peg in hole.}
        \label{fig:supp-2d-peg-in-hole}
    \end{subfigure}
    \hfill
    \begin{subfigure}[b]{0.25\textwidth}
        \centering
        \includegraphics[height=9em]{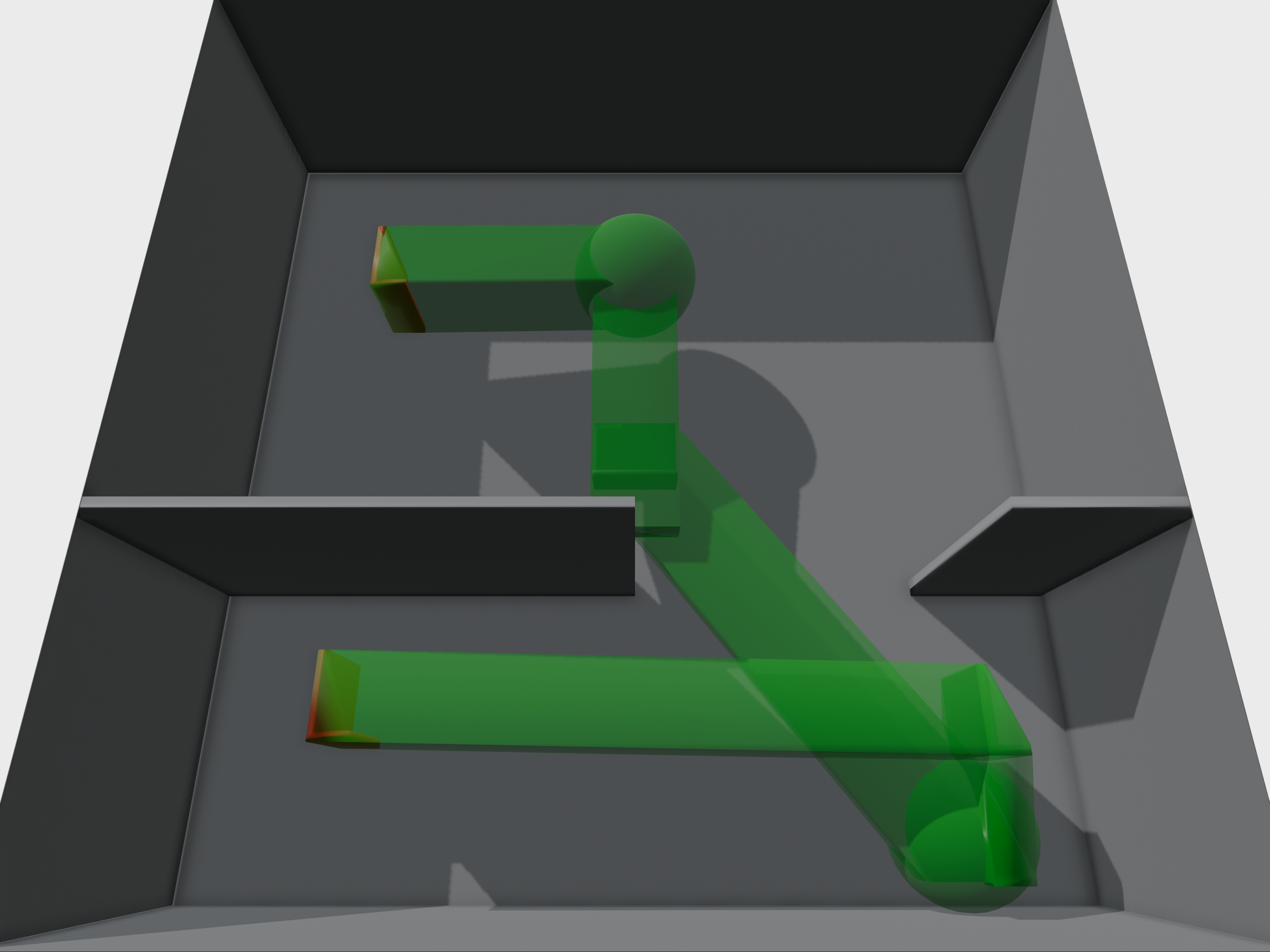}
        \caption{3d piano w.}
    \end{subfigure}
    \hfill
    \begin{subfigure}[b]{0.25\textwidth}
        \centering
        \includegraphics[height=9em]{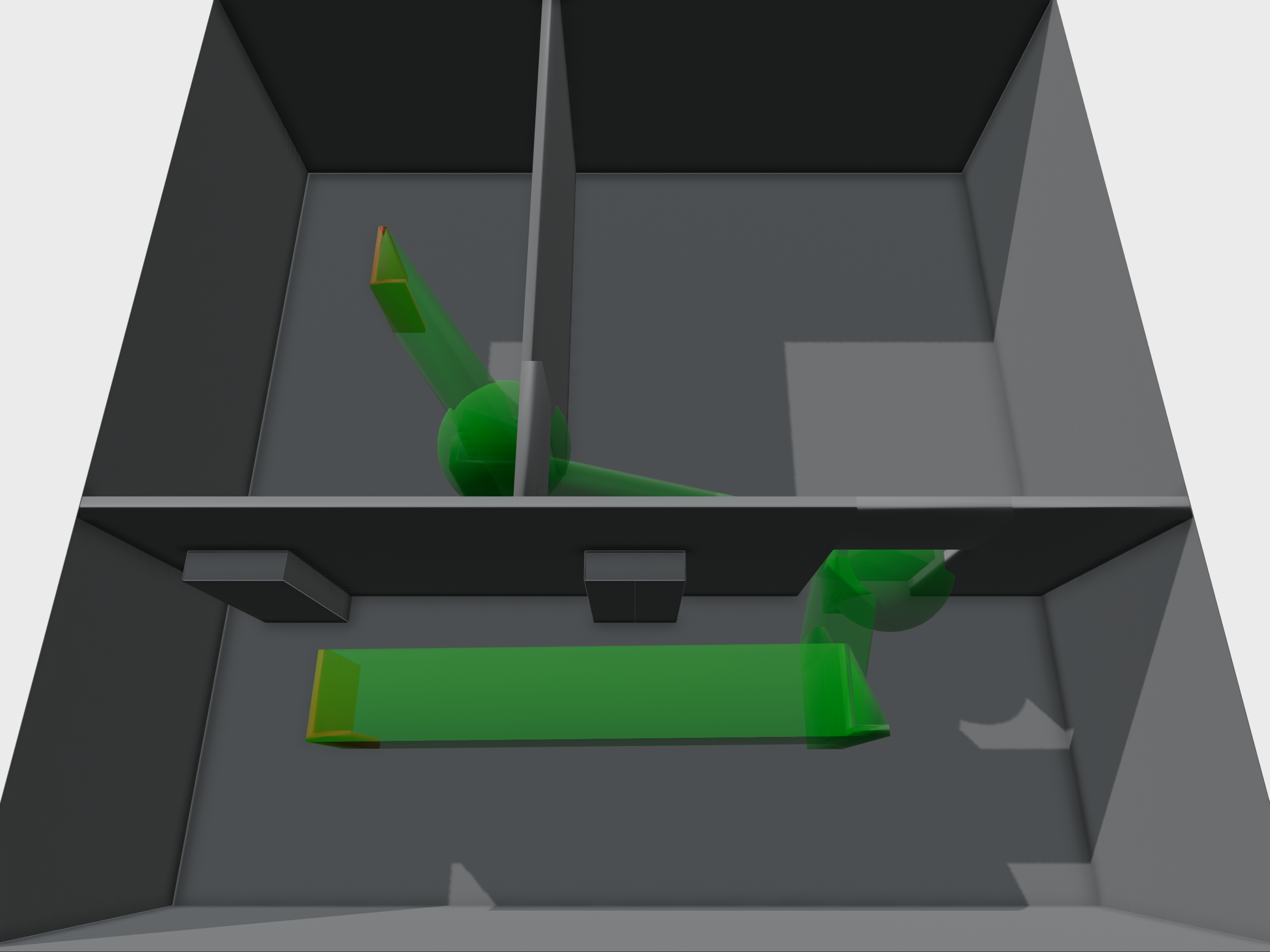}
        \caption{3d-narrow.}
    \end{subfigure}
    \hfill
    \begin{subfigure}[b]{0.27\textwidth}
        \centering
        \includegraphics[height=9em]{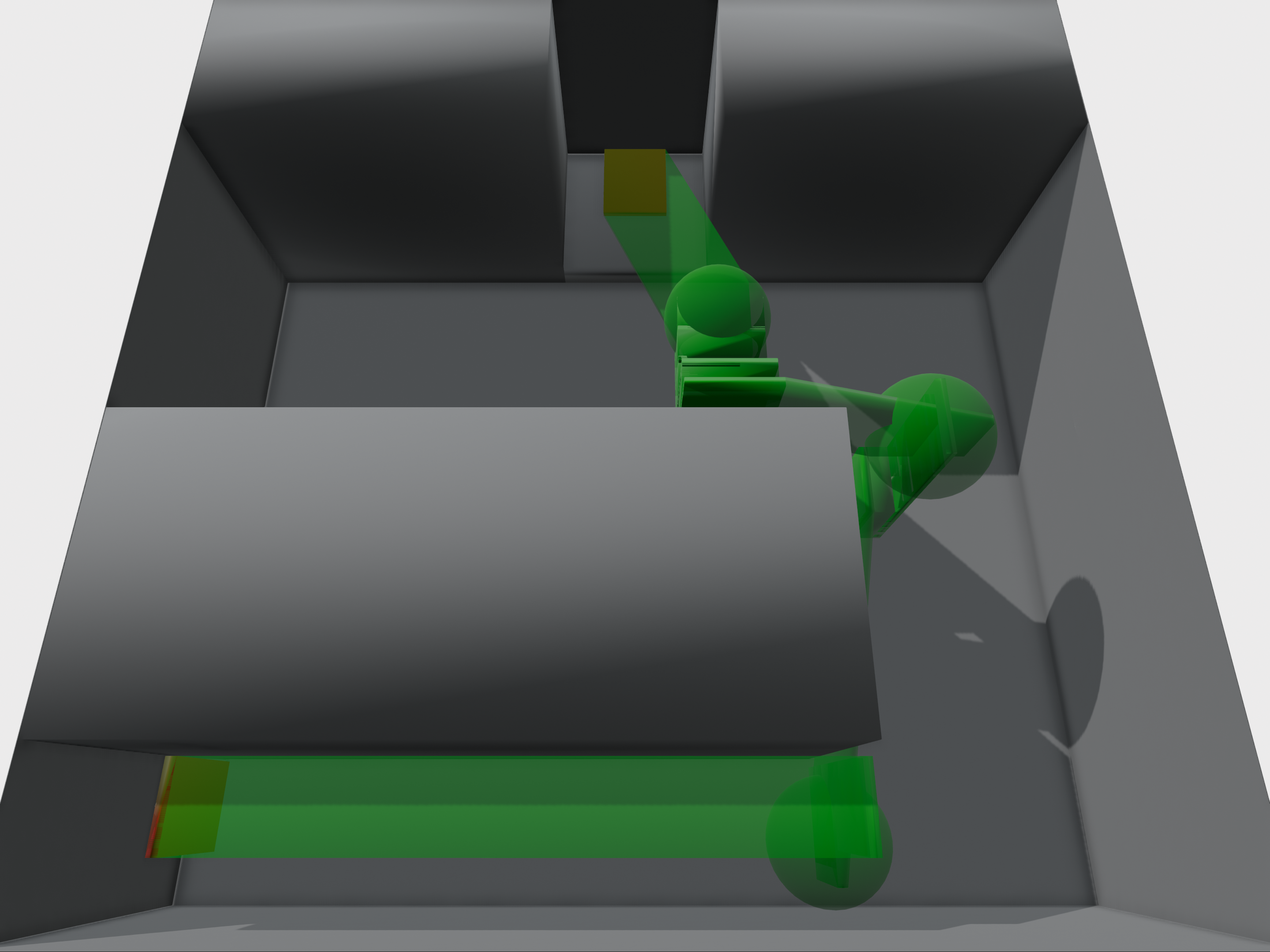}
        \caption{3d-peg-in-hole.}
        \label{fig:supp-3d-peg-in-hole}
    \end{subfigure}
    \caption{Visualization of planning results in 2D (obstacles: pink) and 3D scenarios. The start and end configurations are ploted with red, while waypoints and their corresponding reachable sets are shown in green. In 2D scenarios, we also illustrate the graph $\CoarseGraph$ with vertices ($\CoarseVert$) as blue polytopes and edges ($\CoarseEdge$) as blue lines connecting red dots, which represent the centers of vertex polytopes.~\cref{fig:supp-2d-peg-in-hole,,fig:supp-3d-peg-in-hole} show the planning results of convex objects, while the rest show the results of L-shape non-convex objects.}
    \label{fig:supp-exp-env}
\end{figure*}

When extending the conditions to allow vertices of convex quadrilateral $P_0$ to be contained in more than one of $P_i$, it is equivalent to considering an enlargement of the polytopes $P_1$ and $P_2$. Therefore, the proposition that convex quadrilateral $P_0$ is contained within the union of the polytopes still holds.
\end{proof}

\subsection{Additional Experiment Results}
We first provide additional experiment results in~\cref{fig:supp-exp-env} of unplotted planning shapes in the paper. 

We further include more examples in 3D to discuss our method, shown in~\cref{fig:increase-env-dup}.
For our method in 2D, our solution set will expand if we a) add more convex covers in the workspace, b) increase the
waypoint number inside MILP $N$ or c) increase the fineness of discretization on the surface $\partial P_{ij}$. Moreover, if all the vertice-to-center line segments lie inside the planned object, our method will plan for the exact geometry of the object and can eliminate the missing set of solutions eventually.
In 3D, our current method solves increasingly challenging problems by adding more convex polytopes into the free workspace, as illustrated in~\cref{fig:increase-env-dup}. This results in more nodes being included in the planning graph $G_d$, allowing for more maneuvers in the planned path and enabling the solution of both narrow-passage and more general problems.
We acknowledge that our current 3D method excludes a subset of solutions where most convex covers are so narrow that the object cannot fully rotate within the free space, and only a small set of orientations are admissible.

\begin{figure*}[h]
    \vspace{6pt}
    \centering
    \begin{subfigure}[b]{0.3\textwidth}
        \centering
        \includegraphics[height=9em]{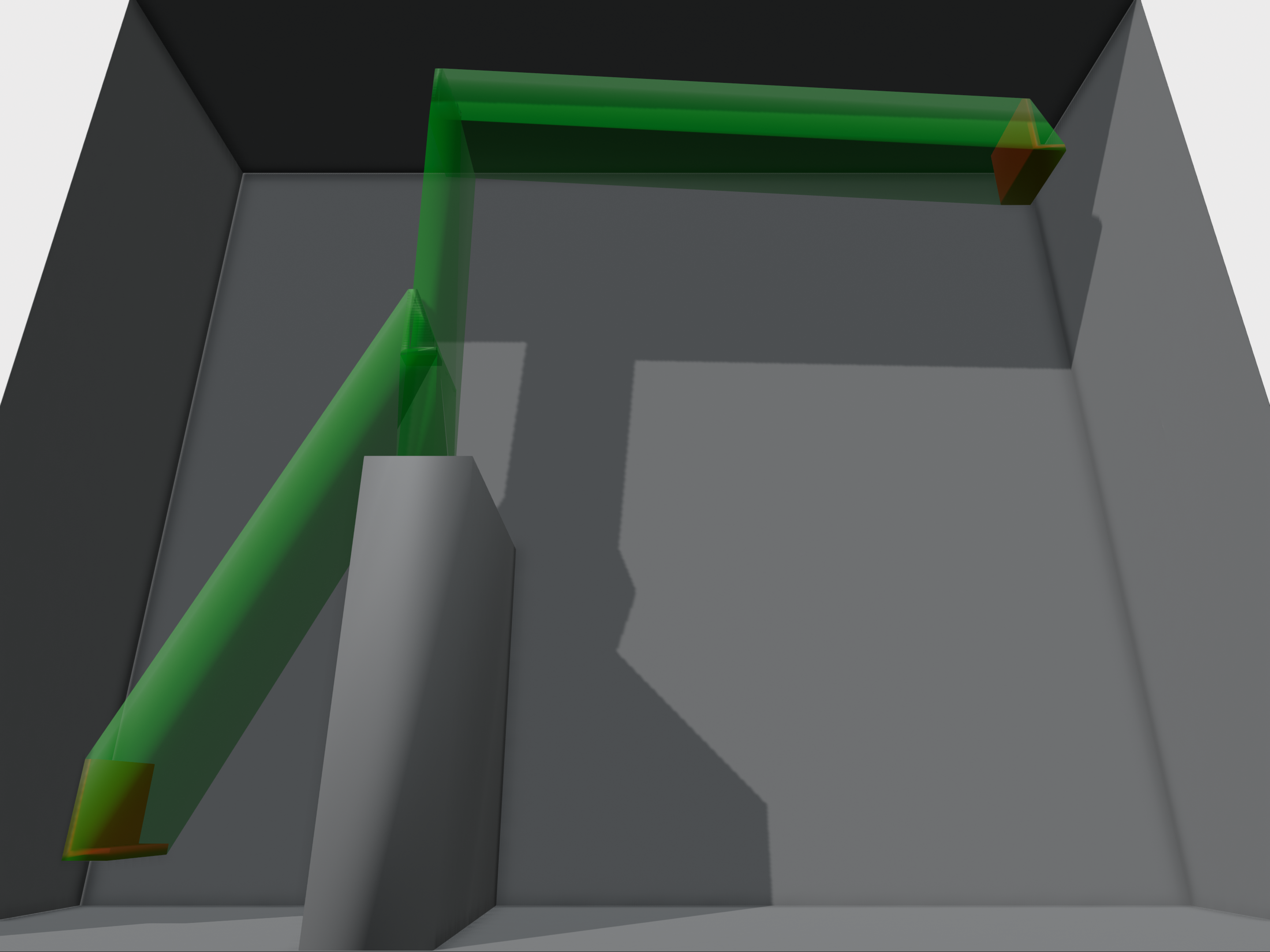}
        \caption{1 obstacle - solution}
    \end{subfigure}
    \hfill
    \begin{subfigure}[b]{0.3\textwidth}
        \centering
        \includegraphics[height=9em]{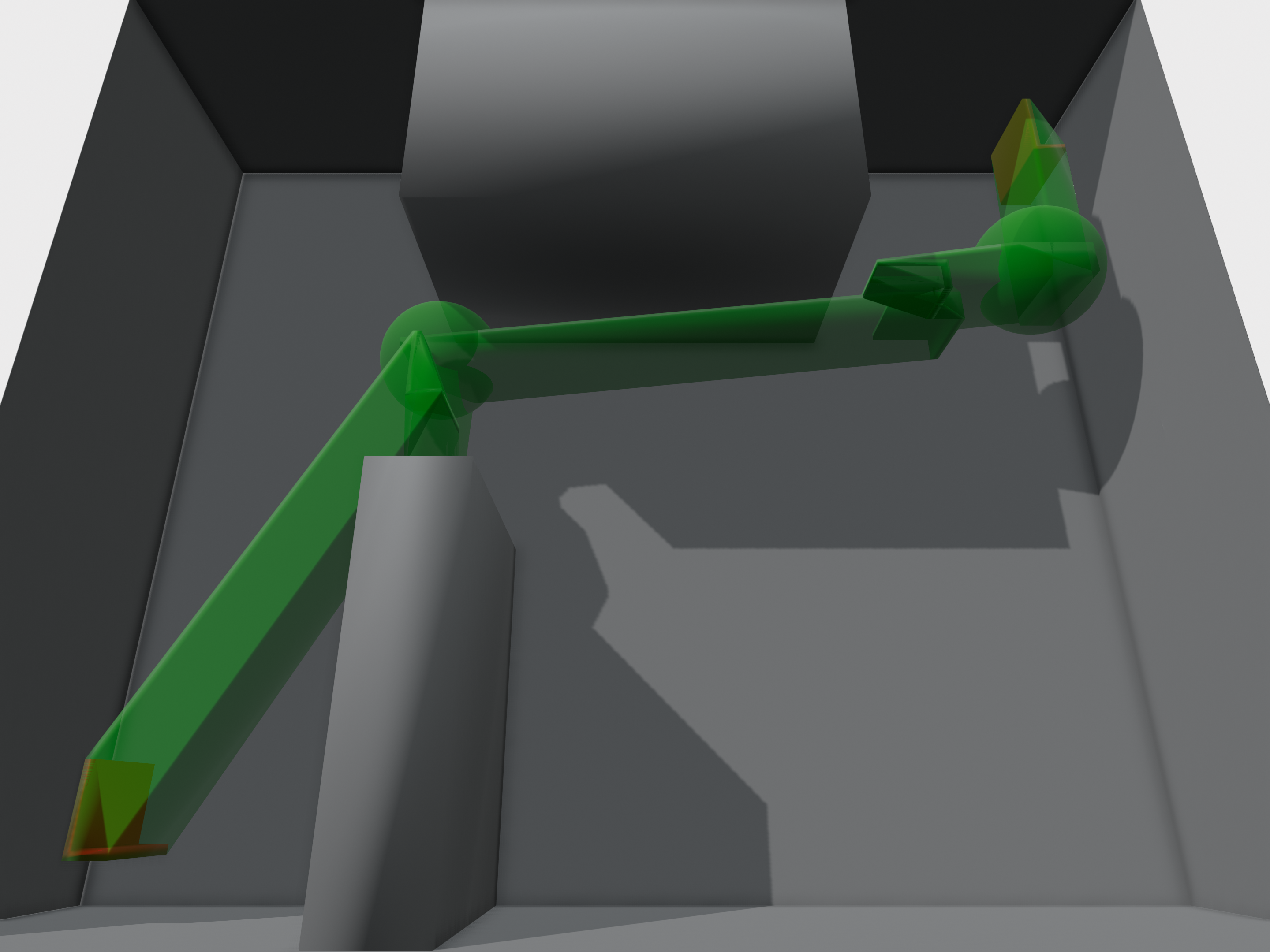}
        \caption{2 obstacles - solution}
    \end{subfigure}
    \hfill
    \begin{subfigure}[b]{0.3\textwidth}
        \centering
        \includegraphics[height=9em]{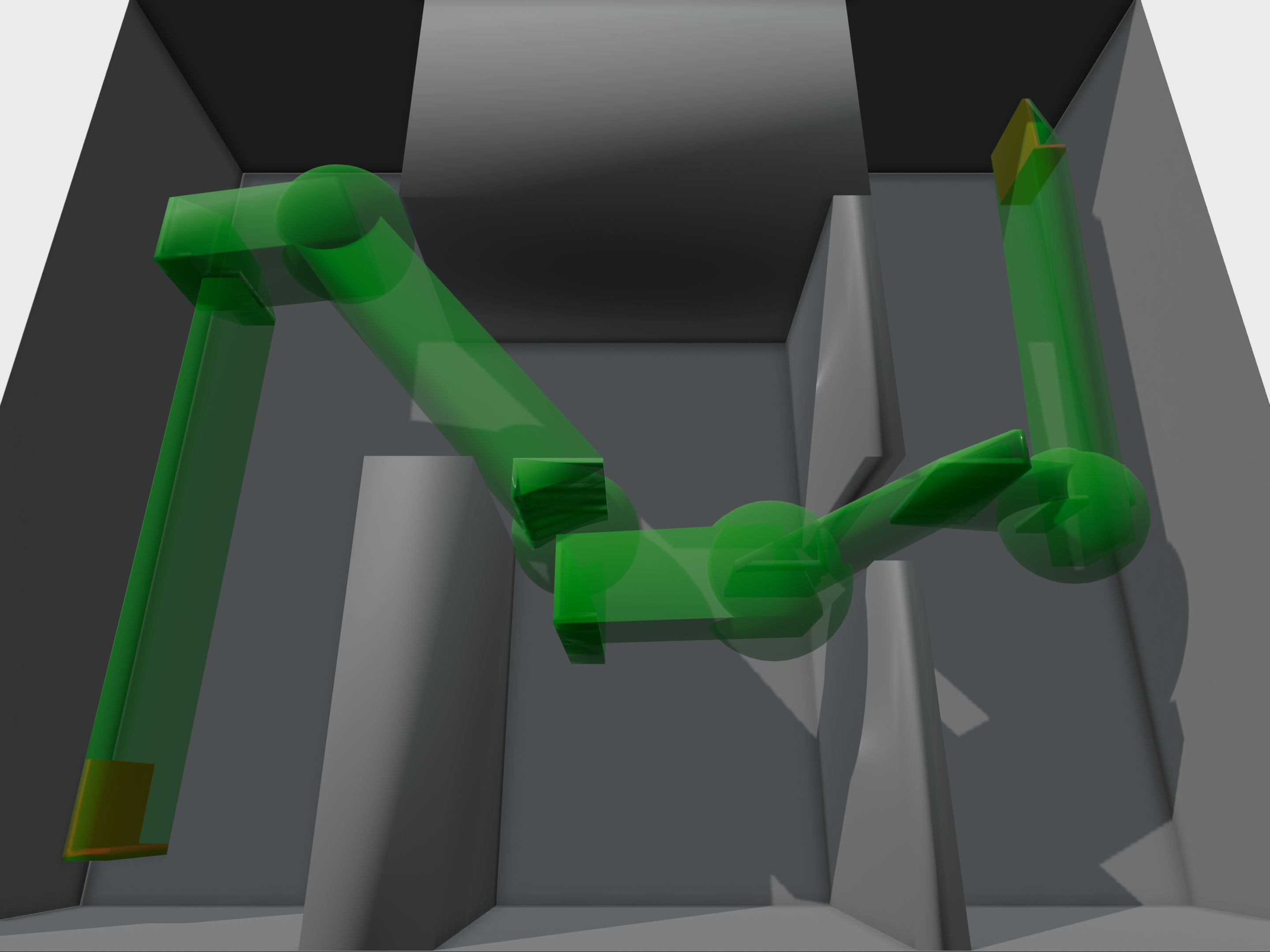}
        \caption{4 obstacles - solution}
    \end{subfigure} \\ \vspace{.5em}
    \begin{subfigure}[b]{0.3\textwidth}
        \centering
        \includegraphics[height=9em]{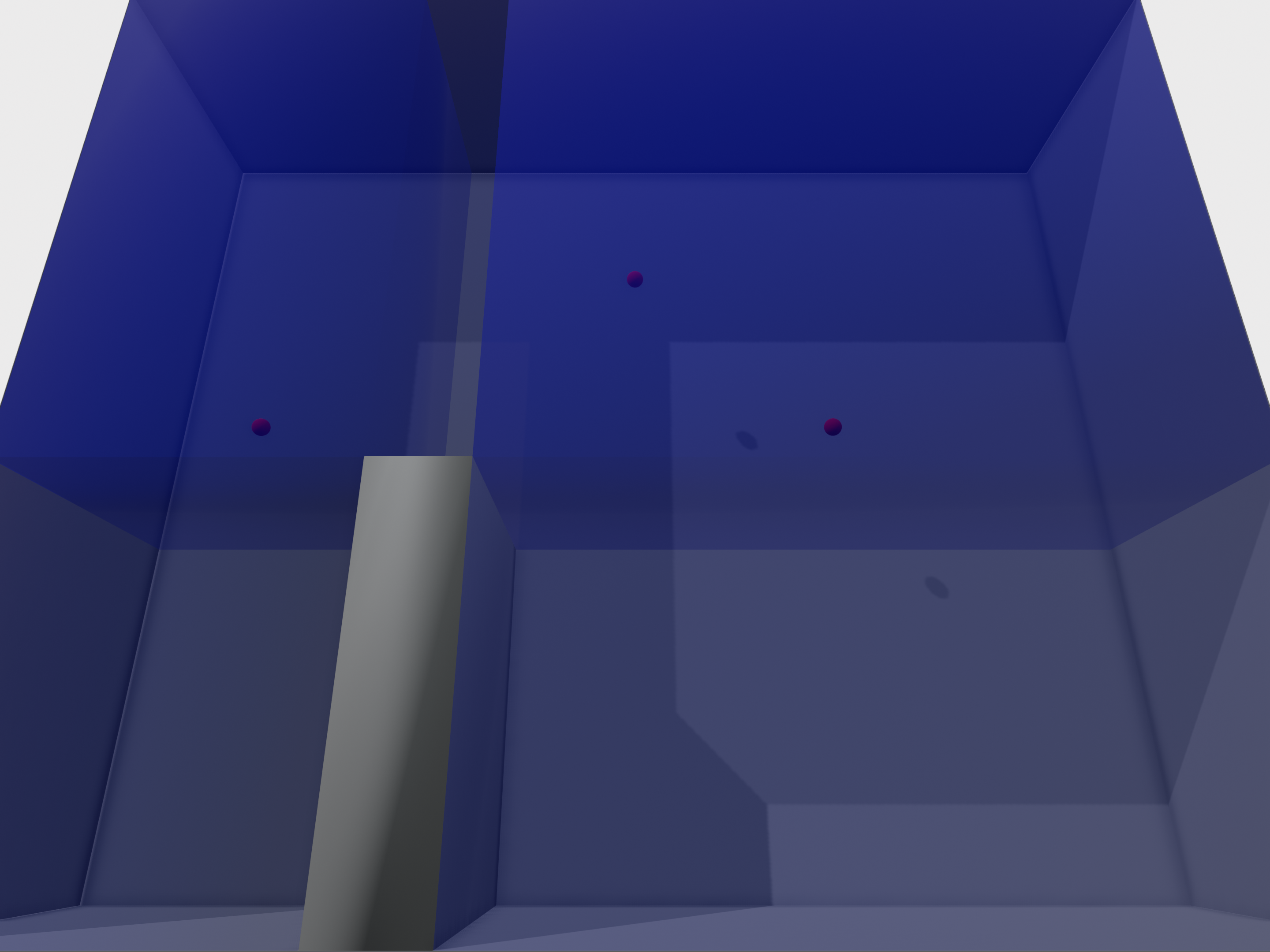}
        \caption{1 obstacle - decomposition}
    \end{subfigure}
    \hfill
    \begin{subfigure}[b]{0.3\textwidth}
        \centering
        \includegraphics[height=9em]{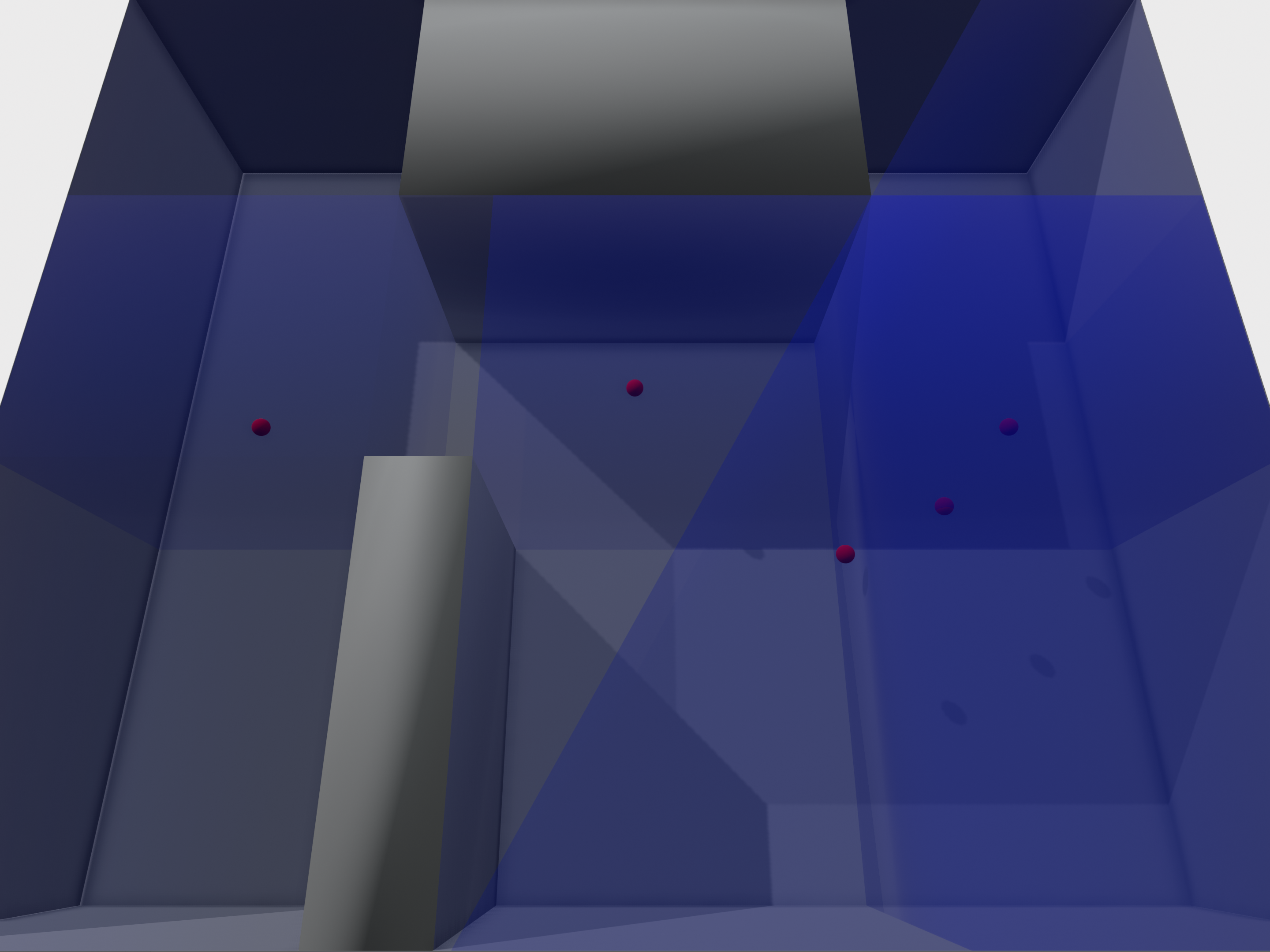}
        \caption{2 obstacles - decomposition}
    \end{subfigure}
    \hfill
    \begin{subfigure}[b]{0.3\textwidth}
        \centering
        \includegraphics[height=9em]{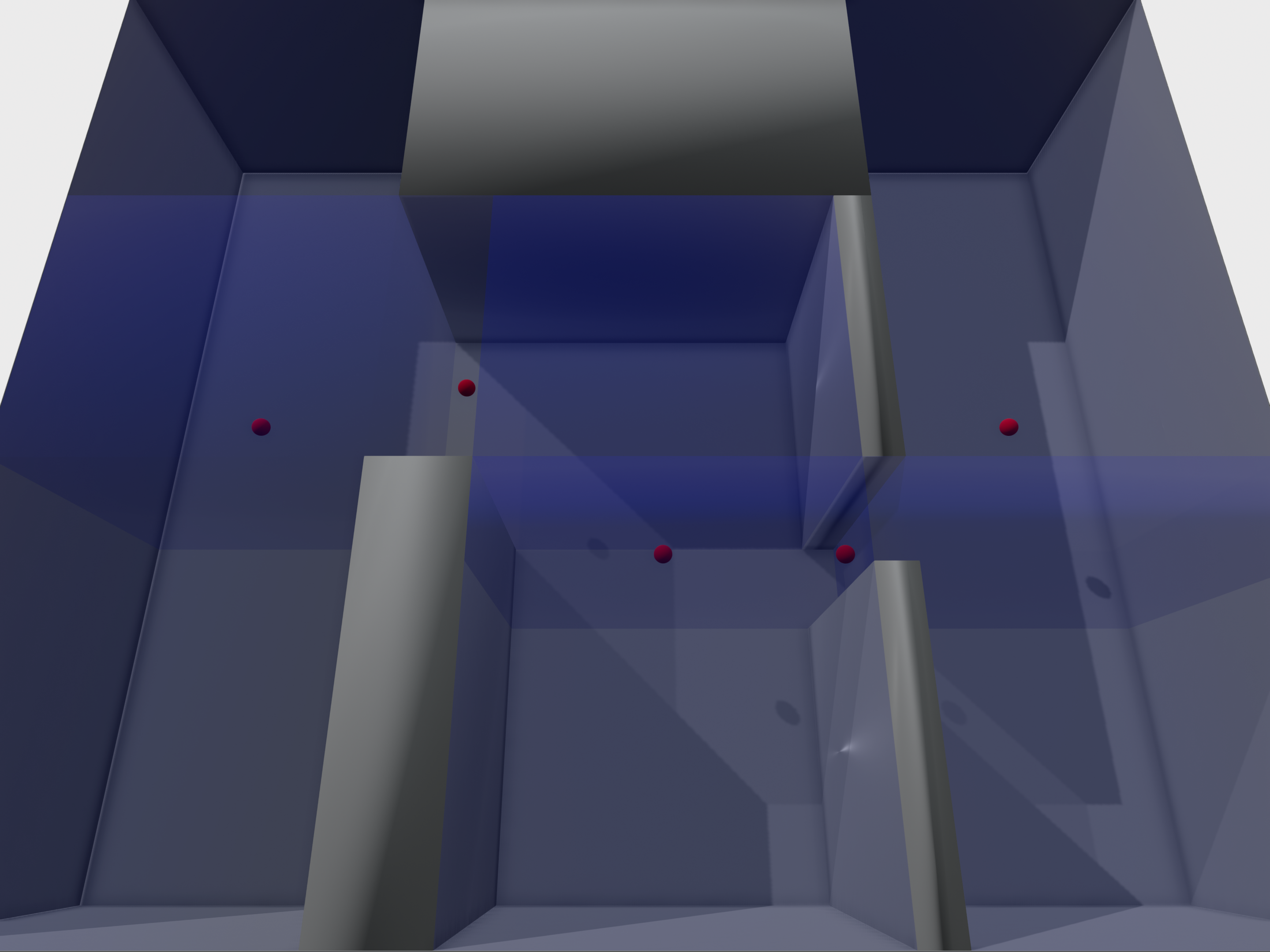}
        \caption{4 obstacles - decomposition}
    \end{subfigure}
    \caption{Three planning problems in 3D with an increasing number of obstacles. New obstacles are added to block the solution path from the previous problem. Our method solves more challenging problems incrementally by adding more polytope covers within the free space, allowing for more maneuvers in the solution paths. While the object is restricted to translational movements in 3D, this simplification enables straightforward collision-avoidance encoding while still permitting the solution of general problems. (a-c) show the solutions to the problem, and (d-f) show the convex polytope covers (blue) inside the free workspace, with a red dot indicating the center of each polytope.}
    \label{fig:increase-env-dup}
\end{figure*}

\subsection{Discussion on Complexity of Discretization}
The fineness of discretization significantly impacts the size of the solution set. If it's too fine, then the MILP problem will be too large. If the discretization is too coarse, we may exclude a large portion of solution set. So we'd like to include a discussion on the complexity of the discretization for our method.

In our 2D experiments, we used the following discretization parameters:
    \begin{itemize}
        \item The rotation $\theta$ is discretized as $N_\theta=12$ values, allowing rotation as small as $30^\circ$.
        \item The translation discretization on $\partial P_{ij}$ is $N_t=60$ points with equal intervals along the line ring.
        \item The discretization in collision avoidance is that each line segment to be checked is divide into $N_c=10$ equal parts.
    \end{itemize}
    The complexity of the discretization is: 
    \begin{itemize}
        \item The discretization of $\theta$ and translation will affect the construction of patches on $\partial P_{ij}$, thus affecting the total number of MILP to be solved. Not considering the grouping operation, the number of MILP will grow at the rate of $O(N_\theta \cdot N_t)$.
        \item Inside each MILP, the discretization parameter $N_c$ in collision avoidance only affects the number of constraints at the rate of $O(N_c)$ but has no effect on the number of variables. The discretization of $\theta$ and translation will affect the number of variables at the rate of $O(N_\theta+N_t)$ without without affecting the number of constraints.
    \end{itemize}

\end{document}